\documentclass{article}

\usepackage{multirow} 
\usepackage{own}

\PassOptionsToPackage{numbers, compress}{natbib}

\usepackage{amsmath,amsfonts,bm}

\def\eqref#1{equation~\ref{#1}}

\def\1{\bm{1}}

\DeclareMathAlphabet{\mathsfit}{\encodingdefault}{\sfdefault}{m}{sl}

\SetMathAlphabet{\mathsfit}{bold}{\encodingdefault}{\sfdefault}{bx}{n}

\usepackage{lipsum}		
\usepackage{amsthm} 
\usepackage[utf8]{inputenc} 
\usepackage[T1]{fontenc}    
\usepackage{url}            
\usepackage{booktabs}       
\usepackage{amsfonts}       
\usepackage{amsmath} 
\usepackage{nicefrac}       
\usepackage{microtype}      
\usepackage{lipsum}	
\usepackage{graphicx}
\usepackage{hyperref}
\usepackage{natbib}
\usepackage{doi}
\usepackage{siunitx}
\usepackage[tableposition=above]{caption}
\usepackage{pifont}
\usepackage{tabularx}
\usepackage{colortbl}
\usepackage{listings}
\usepackage{subfigure}
\usepackage{mathtools}
\usepackage{romannum}
\usepackage{amsmath}
\usepackage{adjustbox}
\usepackage{algorithm}  
\usepackage{algorithmic}  
\usepackage{xcolor}         
\usepackage{amsmath}
\usepackage{amsfonts}
\usepackage{wrapfig}

\definecolor{citecolor}{HTML}{2980b9}
\definecolor{linkcolor}{HTML}{c0392b}
\definecolor{urlcolor}{RGB}{157,49, 251}

\hypersetup{
  colorlinks  = true,
  citecolor   = citecolor,
  linkcolor   = linkcolor,
  urlcolor    = urlcolor,
  breaklinks  = true
}

\usepackage{xcolor} 

\definecolor{graybg}{gray}{0.9}
\newcommand{\defeq}{\vcentcolon=}

\newtheorem{thm}{Theorem}[section]
\newtheorem{prop}[thm]{Proposition}

\newtheorem{cor}[thm]{Corollary}

\newtheorem{assume}[thm]{Assumption}

\newtheorem{rmk}[thm]{Remark}

\newcommand{\clo}[1]{\overline{#1}}
\newcommand{\bb}[1]{\left\{#1\right\}}

\newcommand{\RNum}[1]{\uppercase\expandafter{\romannumeral #1\relax}}
\newcommand{\rnum}[1]{\lowercase\expandafter{\romannumeral #1\relax}}

\newcommand{\R}{\mathbb{R}}

\DeclarePairedDelimiterX{\norm}[1]{\lVert}{\rVert}{#1}

\title{Exploring Time Conditioning in Diffusion Generative Models from Disjoint Noisy Data Manifolds}

\author{
\vspace{-20pt}\\
\textbf{Liuzhuozheng Li}$^{1,2}$ \quad
\textbf{Zhiyuan Zhan}$^{2,7}$ \quad
\textbf{Shuhong Liu}$^{2}$ \quad
\textbf{Dengyang Jiang}$^{3}$ \\[1mm]
\textbf{Zanyi Wang}$^{4}$ \quad
\textbf{Guang Dai}$^{1}$ \quad
\textbf{Jingdong Wang}$^{6}$ \quad
\textbf{Mengmeng Wang}$^{5,1}$\thanks{\parbox[t]{\dimexpr\linewidth-1.8em\relax}{Corresponding author}} \\[2mm]
$^{1}$SGIT AI Lab \quad
$^{2}$UTokyo \quad
$^{3}$HKUST \quad
$^{4}$UCSD \quad
$^{5}$ZJUT \quad
$^{6}$Baidu \quad
$^{7}$RIKEN AIP\\
\\
\textit{Code is available at:} \url{https://github.com/liuzhuozheng-LI/time-uncond-diffusion}
}

\renewcommand{\headeright}{}

\hypersetup{
pdftitle={A template for the arxiv style},
pdfsubject={q-bio.NC, q-bio.QM},
pdfauthor={David S.~Hippocampus, Elias D.~Striatum},
pdfkeywords={First keyword, Second keyword, More},
}

\begin{document}

\maketitle
\definecolor{lightgreen}{RGB}{223, 243, 212}

\begin{abstract}
Practically, training diffusion models typically requires explicit time conditioning to guide the network through the denoising sampling process. Especially in deterministic methods like DDIM, the absence of time conditioning leads to significant performance degradation. However, other deterministic sampling approaches, such as flow matching, can generate high-quality content without this conditioning, raising the question of its necessity. In this work, we revisit the role of time conditioning from a geometric perspective. We analyze the evolution of noisy data distributions under the forward diffusion process and demonstrate that, in high-dimensional spaces, these distributions concentrate on low-dimensional hyper-cylinder-like manifolds embedded within the input space. Successful generation, we argue, stems from the disentanglement of these manifolds in high-dimensional space. Based on this insight, we modify the forward process of DDIM to align the noisy data manifold with the flow-matching approach, proving that DDIM can generate high-quality content without time conditioning, provided the noisy manifold evolves according to the flow-matching method. Additionally, we extend our framework to class-conditioned generation by decoupling classes into distinct time spaces, enabling class-conditioned synthesis with a class-unconditional denoising model. Extensive experiments validate our theoretical analysis and show that high-quality generation is achievable without explicit conditional embeddings.
\end{abstract}

\begin{figure*}[t]
    \vskip -0.1in
    \centering
    \includegraphics[width=0.9\textwidth, height=0.48\textwidth]{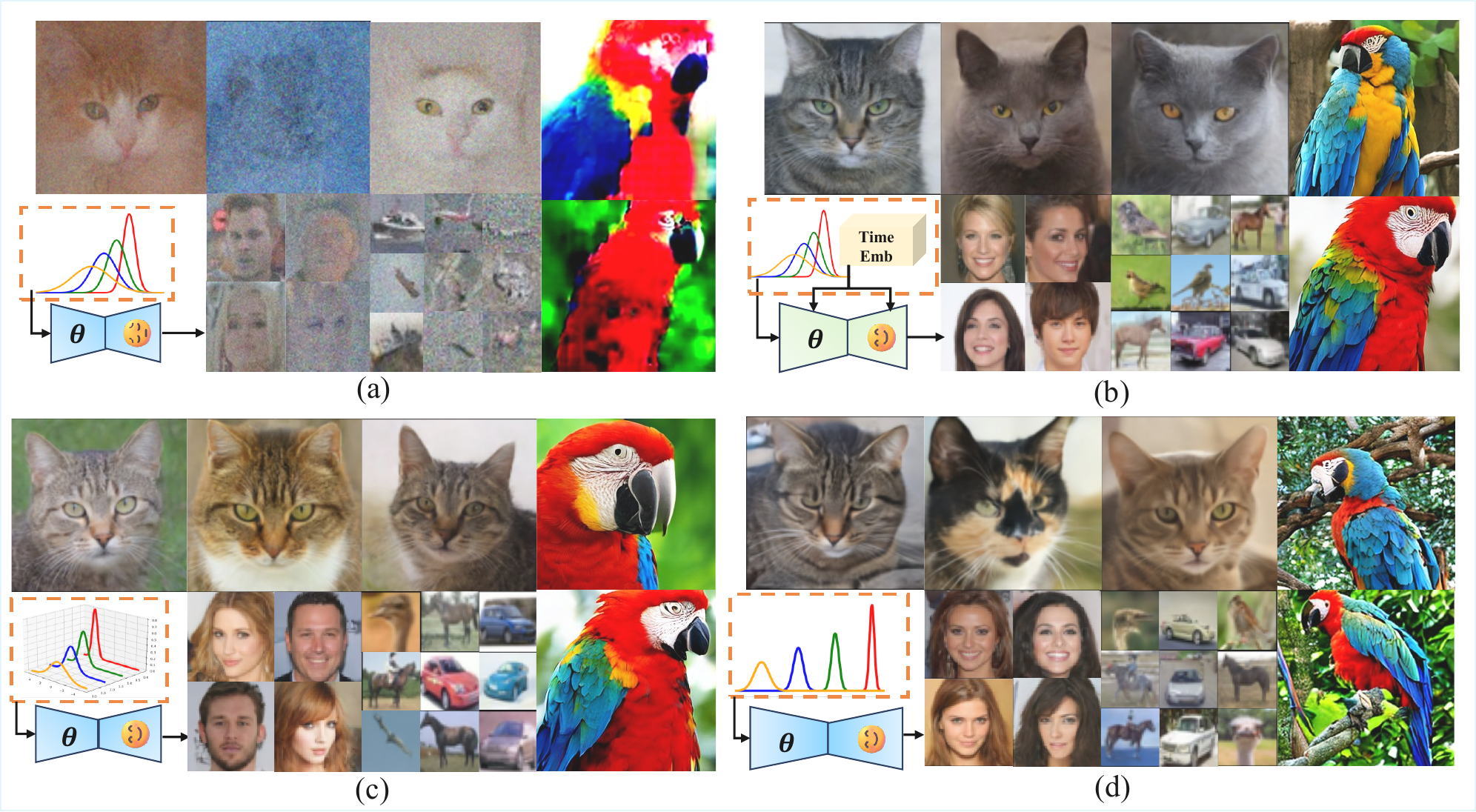}
    \caption{Generated images from AFHQ-Cat \cite{choi2020stargan}, CelebA\cite{liu2015deep}, CIFAR10\cite{Krizhevsky2009LearningML} and ImageNet\cite{deng2009imagenet} dataset using DDIM sampler. (a) shows the failed generation without time conditioning; the generated images are of poor quality and oversaturated. (b),(c) and (d) show the image generation using time-conditioning, disjointed data manifold, and time-space orthogonality methods, respectively.}
    \label{fig:intro}
    \vskip -0.2in
\end{figure*}

\section{Main}
\label{sec:main}

Machine learning seeks to extract statistical structure and useful representations from large-scale data and then use that information for a wide range of downstream tasks \cite{sugiyama2013}. Among its many branches, generative modeling is concerned with learning the underlying data distribution to synthesize new samples. Similarly, just as painting turns observation and imagination into an image on a canvas, modern visual generative models synthesize visual content inside a structured pixel space from the learned training samples. This analogy is particularly natural for diffusion models\cite{ho2020denoising, sohl2015deep}-rather than drawing an image in a single step, they progressively refine a noisy canvas over multiple denoising steps, much like a painter who repeatedly adjusts composition, tone, and local details\cite{li2026thinking}. Teaching a generative model can be seen as teaching a novice painter by repeatedly reconstructing reference works, and similar painter-inspired views have appeared in computer graphics for decades~\cite{Foley-CGP-1995, jiang2025no}. In practical systems, however, this painter-like process is almost always accompanied by an explicit instruction that specifies which denoising step the model is currently performing. That instruction is the timestep embedding.

Designing a generative model remains difficult because real-world data are high-dimensional, structured, and highly multimodal. Generative modeling methods such as VAEs~\cite{kingma2013auto, van2017neural, burgess2018understanding}, GANs~\cite{heusel2017gans}, and more recent score-based and flow-based paradigms~\cite{ho2020denoising, song2020denoising, lipman2022flow, chen2023sampling, DORMAND198019, bortoli2021diffusion} are all designed to recover the training data distribution. Across these families, diffusion-type methods have been used broadly across image, video, and audio generation and editing~\cite{rombach2022high, ho2022video, melnik2024video, schneider2023archisound, li2025refvton, zhu2025flux}. Among them, diffusion models are especially attractive because they provide a stable training objective for approximating an otherwise intractable high-dimensional data distribution $\mathbb{P}_{\mathrm{data}}$~\cite{cao2024survey, wang2025deforming}, while flow matching provides a closely related formulation that learns the transport dynamics along an explicit probability path~\cite{lipman2022flow, chen2023sampling, DORMAND198019, liu2023flow}. In score-based diffusion, data are gradually perturbed toward Gaussian noise, and the model learns the field needed to reverse that corruption; in flow matching, one specifies a path between noise and data and directly regresses the velocity field that transports samples along it. Both viewpoints therefore learn how to generate new samples by following a time-dependent field from a simple source distribution toward the data distribution. They can also be interpreted through the lens of nonequilibrium statistical physics and energy-based modeling~\cite{0Hopfield, 10.5555/1046920.1088696, oksendal2013stochastic, ANDERSON1982313}, and are typically implemented with high-capacity denoisers such as U-Net or DiT-style architectures~\cite{ronneberger2015u, peebles2023scalable}.

In standard score-based diffusion, the score function $\nabla_{\boldsymbol{x}}\log p_t(\boldsymbol{x})$ provides the local denoising direction needed to generate new samples, while in flow matching, the velocity field specifies how samples should move along the chosen probability path~\cite{lipman2022flow, chen2023sampling, song2019generative, song2020sliced}. In other words, the information is time-continuous, and both objects depend on $t$. To obtain the correct prediction, the model must know where the current sample lies along the forward path. This is why modern diffusion and flow-based architectures almost always input time information explicitly, most commonly through sinusoidal positional embeddings, learned timestep embeddings, or rotary-style encodings~\cite{vaswani2017attention, dosovitskiy2020image, su2024roformer}. Recent work has therefore treated time conditioning as both a modeling necessity and an optimization issue, emphasizing interference across timesteps and the difficulty of decoupling their objectives~\cite{go2023addressing, ma2024decouple}. Yet this standard design choice is no longer unquestioned. In particular, \cite{sun2025noise} showed that diffusion-type models can be trained and sampled without explicit time conditioning, and that under suitable forward processes, this can even work surprisingly well. Motivated by such observations, several recent application-oriented and product-scale generative systems have started to adopt time-unconditional pretraining strategies, especially in flow-based large models~\cite{imageteam2025zimageefficientimagegeneration, tang2025exploringdeepfusionlarge}. This trend is practically important because removing timestep embeddings simplifies the denoiser, reduces conditioning overhead, and makes pretraining pipelines cleaner.

This paper asks when such explicit time conditioning is actually unnecessary. Our answer is geometric rather than purely architectural, that if noisy samples from different timesteps occupy sufficiently separated regions of space, then the sample geometry itself already reveals the relevant temporal information. We therefore study time-unconditional diffusion through the geometry of noisy data manifolds, explain when overlap makes learning ambiguous, and develop two remedies: schedule-level manifold disjoint and orthogonal time-space disentanglement. Figure~\ref{fig:intro} illustrates this contrast, showing that naive time-unconditional denoising diffusion implicit models (DDIM)\cite{song2020denoising} sampling suffers a clear degradation in image quality, whereas geometry-aware disjoint strategies achieve performance comparable to the time-conditioned baseline.

\subsection{Background on diffusion model and time conditioning}

We consider data samples $\boldsymbol{x}\in\mathbb{R}^D$, where $D$ denotes the input data dimension. Throughout the paper, $\boldsymbol{x}_0\sim\mathbb{P}_{\mathrm{data}}$ denotes a clean sample, $\boldsymbol{z}\sim\mathcal{N}(\boldsymbol{0},\boldsymbol{I})$ denotes Gaussian noise, and $\boldsymbol{x}_t$ denotes the noisy sample obtained after running the forward diffusion process to timestep $t$. In the denoising diffusion probabilist model (DDPM), $\boldsymbol{x}_t$ is the partially corrupted version of $\boldsymbol{x}_0$ at noise level $t$. To make the role of time practical, let $\tau$ be a timestep random variable on $[0,T]$, sampled uniformly when training pairs are drawn uniformly over timesteps. Conditional on $\tau=t$, the noisy random variable $\boldsymbol{X}$ has density $p(\boldsymbol{x}\mid \tau=t)$. We write this conditional density as $p_t(\boldsymbol{x})$, and for simplicity, we denote $p_t(\boldsymbol{x})=p(\boldsymbol{x}\mid \tau=t)$. The posterior density $p(\tau=t\mid \boldsymbol{x})$, abbreviated as $p(t\mid \boldsymbol{x})$ when no confusion arises, measures how informative a noisy sample is about its generating timestep. This is why the score field $\nabla_{\boldsymbol{x}}\log p_t(\boldsymbol{x}) = \nabla_{\boldsymbol{x}}\log p(\boldsymbol{x} \mid t)$ is a time-conditional object. Different values of $t$ correspond to different conditional densities, so a model that predicts this field must know which conditional law is being queried. Supplying that information to the network is what we call time conditioning.

Score-based diffusion and flow matching differ in objective form, but both can be viewed as learning time-dependent vector fields that transport simple source distributions to the data distribution~\cite{song2019generative, song2020sliced, lipman2022flow, chen2023sampling}. For DDPM/DDIM, we first consider the variance-preserving forward process described in Section~\ref{sec:details}, for which $\boldsymbol{x}_t$ is the VP noisy sample at timestep $t$ and the closed-form forward marginal is
Score-based diffusion and flow matching differ in objective form, but both can be viewed as learning time-dependent vector fields that transport simple source distributions to the data distribution~\cite{song2019generative, song2020sliced, lipman2022flow, chen2023sampling}. For DDPM/DDIM, we first consider the variance-preserving forward process described in Section~\ref{sec:details}. For discrete timesteps $t\in\{1,\dots,T\}$, let $\beta_t\in(0,1)$ denote the variance increment at step $t$, set $\alpha_t=1-\beta_t$, and define the cumulative signal coefficient
\begin{equation}
    \bar{\alpha}_t = \prod_{s=1}^{t}\alpha_s = \prod_{s=1}^{t}(1-\beta_s).
    \label{eq:alpha_bar_def}
\end{equation}
DDIM uses the same forward marginals as DDPM and changes only the reverse sampling rule, so the noisy sample at timestep $t$ has the closed-form marginal
\begin{equation}
\boldsymbol{x}_t
=
\sqrt{\bar{\alpha}_t}\boldsymbol{x}_0
+
\sqrt{1-\bar{\alpha}_t}\boldsymbol{z},
\qquad
\boldsymbol{z}\sim\mathcal{N}(\boldsymbol{0},\boldsymbol{I}).
\label{eq:forward_process}
\end{equation}
For flow matching, one instead chooses an explicit probability path, often the linear interpolation between data and noise, and learns the associated velocity field directly~\cite{lipman2022flow, chen2023sampling, DORMAND198019}. In both cases, the core object remains a family of timestep-indexed noisy marginals~\cite{chefer2026self, wan2025wan}. To keep the later discussion independent of a specific sampler, we write the forward process in the unified form
\begin{equation}
    \boldsymbol{x}_t = c_t\boldsymbol{x}_0 + \sigma_t\boldsymbol{z},
    \label{eq:forward_unified}
\end{equation}
where $c_t$ is the signal and $\sigma_t$ is the noise coefficient. For the VP process as shown in \eqref{eq:forward_process}, $c_t=\sqrt{\bar{\alpha}_t}$ and $\sigma_t=\sqrt{1-\bar{\alpha}_t}$; for flow matching, these coefficients follow linear interpolation~\cite{song2020denoising, lipman2022flow, esser2024scaling}. Under either formulation, generation requires evaluating a time-indexed score or velocity field at the current location along the path.

\subsection{The standard view and its challenge: why time conditioning became fundamental}

Time conditioning for diffusion models has been a consequence of the underlying dynamics rather than as an optional architectural trick. In early score-based formulations, each time $t$ defines a different marginal $p(\boldsymbol{x}\mid t)$, so the reverse process necessarily depends on the corresponding score field. DDPM made this dependence explicit at discrete timesteps, DDIM retained it in deterministic sampling, and flow matching preserved the same principle through a time-indexed velocity field. As model architectures evolved from convolutional U-Nets to transformer-based generators, the implementation strategy remained, providing the network with an explicit temporal embedding so it can distinguish among timesteps. In practice, this temporal signal is usually encoded through sinusoidal positional embeddings, learned timestep embeddings, or rotary variants~\cite{ho2020denoising, song2020denoising, vaswani2017attention, dosovitskiy2020image, su2024roformer, song2020improved, dosovitskiy2020image}.

The direct removal of time conditioning is not obviously consistent with the theory of score-based diffusion. The reverse process in DDPM or probability-flow DDIM requires a \emph{time-indexed} score field $\nabla_{\boldsymbol{x}}\log p(\boldsymbol{x} \mid t)$ because each timestep corresponds to a different noisy marginal. If the network is trained without timestep input, then in the overlapping region of two noisy marginals $p(\boldsymbol{x}\mid t)$ and $p(\boldsymbol{x}\mid t+1)$, the same noisy sample can correspond to different denoising targets at different timesteps. A single time-unconditional predictor is then forced to compromise between incompatible targets, rather than matching the correct target for the timestep that generated the sample. In other words, once noisy marginals overlap, the model is no longer learning the score field required by the reverse dynamics, but a mixture of incompatible fields. This mismatch is exactly the type of multi-task conflict emphasized by \cite{go2023addressing, ma2024decouple}. It also explains why simply forcing a network to infer the timestep from the current noise level is theoretically fragile; that inference is only valid when the geometry of noisy data makes $p(t\mid \boldsymbol{x})$ sharply concentrated \cite{sun2025noise}.

Recent work, however, has started to relax this standard view~\cite{zhang2023practical}. As reported by \cite{sun2025noise}, removing time conditioning from DDIM leads to poor generation quality, while time-unconditional generation can work much better in flow matching. This provides a strong clue that the question is not merely whether the model \emph{can} learn $p(\boldsymbol{x}\mid t)$ without time input, but under what geometric conditions the noisy data still carry enough information to identify the correct denoising direction. Our answer is that time conditioning is not fundamentally about giving the network an abstract symbol $t$; rather, it is about ensuring that different noisy marginals are distinguishable. If the geometry of noisy data already makes the timestep recoverable from $\boldsymbol{x}_t$ itself, then explicit time embeddings become redundant. If not, removing them creates an ill-posed learning problem. This is exactly why several recent large-scale systems revisit weaker or absent time conditioning while relying on carefully chosen forward paths or training strategies~\cite{sun2025noise, imageteam2025zimageefficientimagegeneration, tang2025exploringdeepfusionlarge}.

\subsection{Why Time Conditioning Can Become Unnecessary: A Geometric Perspective on Diffusion Manifolds}

The standard view treats time information as indispensable because the reverse field changes with $t$. Our perspective is slightly different as the explicit time conditioning becomes unnecessary only when the noisy sample already contains enough geometric information to reveal which timestep generated it. If that inference is reliable, then the model does not need an external time token; if it is unreliable, removing time input creates an ill-posed learning problem.

For real datasets, the manifold hypothesis suggests that clean data occupy a low-dimensional submanifold of the ambient space~\cite{bengio2013representation}, which has been tested both from theoretical perspective~\cite{fefferman2016testing,narayanan2010sample} and experimental perspective~\cite{brown2023verifying,jang2023geometrically,huh2023isometric}. This viewpoint has become increasingly influential in understanding why diffusion models remain effective in very high dimensions despite the curse of dimensionality~\cite{bronstein2021geometric,oko2023diffusion,tang2024adaptivity,boffi2025shallow} and manifold overfitting concerns~\cite{loaiza-ganem2022diagnosing,loaiza-ganem2024deep}. Studies have shown that diffusion can enjoy favorable scaling under manifold assumptions~\cite{debortoli2022convergence,oko2023diffusion,chung2024cfg++, chung2022diffusion}, while many recent works analyze or exploit the geometric sensitivity of DDIM-like samplers for inversion, editing, denoising, and guidance~\cite{chung2022improving, he2023manifold, zhan2026understanding, meilua2024manifold}. These developments make manifold geometry a natural language for asking when temporal information is already encoded in the sample itself.

From this perspective, diffusion data manifolds can be regarded as the manifolds associated with different timesteps, corresponding to the noisy data distributions as the forward process gradually departs from the clean data manifold. Under the manifold hypothesis, one may approximate the clean data $\boldsymbol{x}_0$ as concentrating near a linear subspace of dimension $d'\ll D$. For the VP process in \eqref{eq:forward_process}, the noisy sample $\boldsymbol{x}_t=\sqrt{\bar{\alpha}_t}\boldsymbol{x}_0+\sqrt{1-\bar{\alpha}_t}\boldsymbol{z}$ then combines a low-dimensional signal component with isotropic ambient noise \cite{chung2022improving, he2023manifold}. Existing geometric analyses show that in high dimensions, such noisy samples concentrate near a cylinder-like hypersurface, namely a $(D-1)$-dimensional manifold surrounding the drifted data subspace~\cite{loaiza-ganem2022diagnosing, he2023manifold, yao2025manifold}. For the present discussion, the key point is simply that different timesteps generate different such hypersurfaces.

Our first method uses a schedule-level manifold disjoint. We adjust the forward-process parameters in the diffusion model so that the noisy manifolds associated with adjacent timesteps stay disjoint over a wider range of the trajectory. When this happens, the geometry of $\boldsymbol{x}_t$ itself is enough to identify its timestep, which makes explicit time input unnecessary. Our second method uses an orthogonal time-space disentanglement. Instead of relying only on the separation created by the noise schedule, we encode temporal variation along a direction that is orthogonal to the diffusion space. This separates manifolds from different timesteps by construction and gives the model a geometric way to recover time information even without standard timestep embeddings. The next section formalizes both mechanisms and explains when each one removes the ambiguity caused by overlapping noisy manifolds.

\subsection{Related work and our position}

Recent work on timestep learning has emphasized that diffusion training is inherently multi-task across noise levels. \cite{go2023addressing, ma2024decouple} analyze interference across timesteps and propose training strategies that reduce gradient conflict or decouple objectives. Their focus is on improving optimization while retaining explicit timestep information. By contrast, our question is when timestep input can be removed altogether, and our answer is geometric rather than purely optimization-based. 

The papers most directly related to our setting are those on time-unconditional or weakly time-conditioned generation. \cite{sun2025noise} shows that removing time conditioning degrades DDIM substantially, while time-unconditional generation can be much more successful in flow matching. Recent large-scale systems~\cite{imageteam2025zimageefficientimagegeneration, tang2025exploringdeepfusionlarge} also report that weak or absent time conditioning can be practical in flow-style pretraining. These works establish the empirical phenomenon, but they do not identify a geometric criterion that predicts when removing time information should succeed and when it should fail. Another closely related line studies the geometry of noisy samples and data manifolds in diffusion and DDIM. \cite{loaiza-ganem2022diagnosing} diagnose manifold effects in diffusion models, while \cite{chung2022improving, he2023manifold, yao2025manifold} show that inversion, editing, denoising, guidance, and model mismatch are all sensitive to how noisy samples sit relative to the underlying data manifold. These results suggest that geometry strongly affects diffusion behavior, but they do not formulate time-unconditional diffusion itself as a problem of overlap between timestep-indexed noisy manifolds.

Our paper connects these directions. Relative to the optimization papers, we study identifiability rather than gradient balancing. Relative to the time-unconditional literature, we explain the DDIM failure mode observed by \cite{sun2025noise} through geometric overlap. Relative to the manifold literature, we use geometry not only to analyze sampling behavior but also to derive two concrete remedies, namely schedule-level manifold disjoint and orthogonal time-space disentanglement.

We summarize our contributions below.
\begin{itemize}
    \item We clarify why recent time-unconditional diffusion methods, though seemingly inconsistent with score-based diffusion theory, can still perform well at scale. The key factor is not the explicit inclusion of a timestep token, but whether $p(t \mid \boldsymbol{x})$ is geometrically concentrated. From this perspective, we provide a manifold-based explanation for the gap observed by \cite{sun2025noise}: time-unconditional DDIM fails due to overlapping neighboring noisy data manifolds, whereas flow-based trajectories remain sufficiently disjoint.
    \item We show that under a linear-manifold approximation, noisy data concentrate on thick cylinder-like manifolds\cite{he2023manifold}, derive explicit separation conditions for adjacent shells, and highlight why the required separation becomes more stringent at larger noise levels.
    \item Based on this analysis, we propose two practical ways to realize time-unconditional diffusion: schedule-level manifold disjoint and orthogonal time-space disentanglement. We further extend the same geometric idea to embedding-free conditional generation by separating classes into different geometric directions~\cite{ho2022classifier, dhariwal2021diffusion}. Under these modifications, the DDIM setting that previously failed in \cite{sun2025noise} becomes workable as well.
\end{itemize}
\section{Time-Unconditioning Diffusion}
\label{sec:time_disentangled_diffusion}

This section develops the central geometric claim of the paper: time-unconditional diffusion works when noisy data manifolds at different timesteps are sufficiently disjoint.
We first analyze the geometry of noisy data under a linear-manifold approximation, then derive a practical non-overlap condition for adjacent noisy shells, next compare several forward schedules through this criterion, and finally show how orthogonal time-space disentanglement enforces disentanglement without modifying the original variance-preserving schedule.

\subsection{Geometry of noisy data manifolds}

For simplicity, we assume that the input data $\boldsymbol{x}_0$ have been normalized so that their mean is $\boldsymbol{0}$. Under the manifold hypothesis, real data occupy only a small portion of the ambient space $\mathbb{R}^D$ and concentrate near a low-dimensional manifold~\cite{bengio2013representation, bronstein2021geometric,stanczuk2022your,fefferman2016testing}. Building on the notation introduced in Section~\ref{sec:main}, let $d'$ denote the intrinsic dimension of the clean data manifold. To make the geometry explicit, we adopt the assumption that the data manifold is a $d'-$dimensional linear subspace, which is commonly used in manifold-based analyses of diffusion~\cite{chung2022improving, he2023manifold, yao2025manifold}. 

\begin{assume}\label{assum1}
     Since the data mean has been shifted to the origin, after an affine transformation of coordinates, clean data $\boldsymbol{x}_0$ lie on $\mathcal{M}_0$:
\begin{equation}
    \mathcal{M}_0 := \left\{ \boldsymbol{x} \in \mathbb{R}^D : x_{d'+1}=x_{d'+2}=\cdots=x_D=0 \right\},
\label{eq:clean_manifold}
\end{equation}
where $d'\ll D$, the first $d'$ coordinates span the intrinsic directions of the manifold, and the remaining $D-d'$ coordinates are normal to it, as shown in Figure \ref{fig:manifold}(b).
\end{assume}

This assumption immediately induces an orthogonal decomposition of the full data space,
\begin{equation}
    \mathbb{R}^{D} = \mathbb{R}^{d'} \oplus \mathbb{R}^{D-d'},
\end{equation}
where $\mathbb{R}^{d'} \cong \mathcal{M}_0$ is aligned with the intrinsic directions of the clean manifold and $\mathbb{R}^{D-d'}$ collects the normal directions.
For a clean sample after affine transformation $\boldsymbol{x}_0=(\boldsymbol{x}^{1:d'}_0,\boldsymbol{0})$ with $\boldsymbol{x}^{1:d'}_0\in\mathbb{R}^{d'}$ is the first $d'$ free dimension in \eqref{eq:clean_manifold}, and the VP forward process in \eqref{eq:forward_process} can be written as
\begin{equation}
    \boldsymbol{x}_t
    =
    (\sqrt{\bar{\alpha}_t}\,\boldsymbol{x}^{1:d'}_0 + \sigma_t \boldsymbol{z}_{\parallel},\; \sigma_t \boldsymbol{z}_{\perp}),
\end{equation}
where $\sigma_t=\sqrt{1-\bar{\alpha}_t}$, $\boldsymbol{z}_{\parallel}\in\mathbb{R}^{d'}$, and $\boldsymbol{z}_{\perp}\in\mathbb{R}^{D-d'}$ are standard Gaussian components.
This decomposition is the key to the whole paper.
The tangential part controls motion \emph{along} the clean manifold, while the normal part controls the distance \emph{away from} it.
The timestep information that matters for time-unconditional generation is encoded primarily in this normal distance.

\begin{prop}\label{prop:noisy_data_manifold}
Under Assumption~\ref{assum1}, noisy data $\bm{x}_t$ generated by the VP forward process concentrate around the data manifold $\mathcal{M}_t$ defined by
\begin{equation}
\begin{aligned}
\mathcal{M}_t
&:= \left\{
\boldsymbol{x} \in \mathbb{R}^D :
d\bigl(\boldsymbol{x}, \mathcal{M}_0\bigr)
=
r(t)
\right\},
\end{aligned}
\label{eq:noisy_manifold}
\end{equation}
where $r(t)=\sigma_t\sqrt{(D-d')}$. More precisely, for any $\varepsilon > 0$,
\begin{equation}
    \mathbb{P}\Bigl(r(t)\sqrt{1-2\sqrt{\varepsilon}} \leq d(\boldsymbol{x}_t, \mathcal{M}_0) \leq r(t)\sqrt{1+2\sqrt{\varepsilon}+2\varepsilon}\Bigr)
    \geq
    1 - 2e^{-2(D-d')\varepsilon}.
    \label{eq:thickness}
\end{equation}
\end{prop}
\begin{rmk}
   Note that $\mathcal{M}_t \subset \R^D$ is a $(D-1)$-dimensional hyper-cylinder shell, whose radius increases in $t$, as illuastrated in Figure \ref{fig:manifold}(a).
\end{rmk}

\begin{figure*}[t]
    \centering
    \includegraphics[width=0.86\textwidth, height=0.4\textwidth]{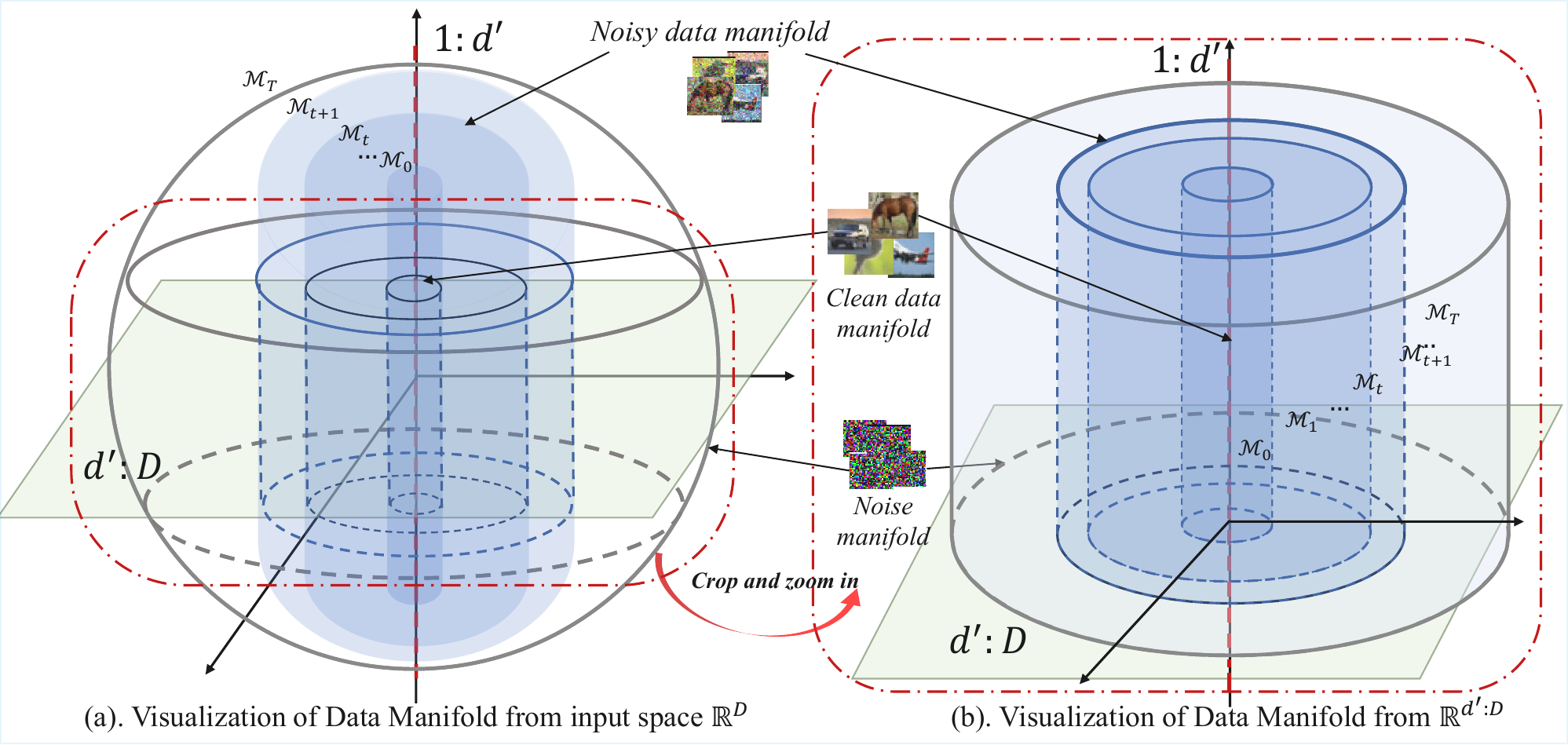}
    \caption{Visualization of the noisy data manifold in $\mathbb{R}^D$. The clean data lie on a low-dimensional linear subspace ($1:d'$), while noisy data form a cylinder-like shell around it in the normal directions ($d':D$).}
    \label{fig:manifold}
\end{figure*}

The proposition identifies the geometry around which $p(\boldsymbol{x}\mid t)$ concentrates.
Indeed,
\begin{equation}
    d\bigl(\boldsymbol{x}_t,\mathcal{M}_0\bigr)
    =
    \sigma_t\,\|\boldsymbol{z}_{\perp}\|_2.
\end{equation}
Since $\boldsymbol{z}_{\perp}\sim \mathcal{N}(\boldsymbol{0},\boldsymbol{I}_{D-d'})$, its norm concentrates sharply around $\sqrt{D-d'}$ in high dimensions.
Equivalently, the normal Gaussian mass concentrates on a thin sphere in $\mathbb{R}^{D-d'}$.
After lifting this sphere back to the full data space and allowing free variation along the intrinsic coordinates $\mathbb{R}^{d'}$, the probability mass forms a thin cylinder-like shell around the clean manifold.
This is the geometric origin of the concentric circles shown schematically in Figure~\ref{fig:method} and empirically in Figure~\ref{fig:toydataset}: if we suppress the intrinsic directions and only look at the normal space $\mathbb{R}^{D-d'}$, each timestep corresponds to a different spherical shell, which can be visualized as a circle in those figures.

This viewpoint also clarifies why the noisy sample can implicitly reveal its own timestep.
If the shells corresponding to different $t$ are disjoint, then knowing where $\boldsymbol{x}_t$ lies in the normal space is enough to identify which noisy manifold generated it.
The denoiser does not need an external time token because the sample geometry already contains the relevant time information.
If those shells overlap, however, the same noisy sample is compatible with multiple timesteps, and time-unconditional learning becomes ambiguous.

\subsection{Distribution mixing explains the need for time conditioning}

The ideal noisy data manifolds $\mathcal{M}_t$ are infinitely thin geometric objects, but real distributions are not. Because Gaussian concentration has a nonzero width, the high-probability region of $p(\boldsymbol{x}\mid t)$ should be regarded as a thick shell around $\mathcal{M}_t$ rather than as the manifold itself, as shown in Proposition~\ref{prop:noisy_data_manifold}. More precisely, for a given $\varepsilon$, the shell $\mathcal{A}_t^\varepsilon$ is defined as
\begin{equation}
    \mathcal{A}_t^\varepsilon \defeq \bb{ \bm{x} \colon r_-(t) \leq d(\boldsymbol{x}, \mathcal{M}_0) \leq r_+(t)},
    \label{eq:A_t_varepsilon}
\end{equation}
where
\begin{equation}
    r_-(t)=r(t)\sqrt{1-2\sqrt{\varepsilon}},
    \qquad
    r_+(t)=r(t)\sqrt{1+2\sqrt{\varepsilon}+2\varepsilon},
\label{eq:thickness_bounds}
\end{equation}
and $r(t)=\sqrt{(1-\bar{\alpha}_t)(D-d')}$ is the shell radius. Proposition \ref{prop:noisy_data_manifold} shows that noise data $\bm{x}_t$ generated by the VP process concentrates on the radial shell $\mathcal{A}_t^\varepsilon$ with high probability.

Moreover, the radial shell width is
\begin{equation}
    w(t)=r_+(t)-r_-(t)
    =
    \sigma_t\sqrt{D-d'}\Bigl(\sqrt{1+2\sqrt{\varepsilon}+2\varepsilon}-\sqrt{1-2\sqrt{\varepsilon}}\Bigr).
    \label{eq:shell_width}
\end{equation}
~\eqref{eq:shell_width} makes two practically important facts explicit.
First, the shell becomes thicker as $\sigma_t$ increases, so later timesteps are intrinsically harder to disentangle.
Second, the thickness of the shell increases linearly in proportion to the radius, and simply increasing the $D-d'$ is insufficient to avoid the distribution mixing." As shown in the rightmost column of Figure~\ref{fig:toydataset}, the shell thickness grows with the dimension $D-d'$. Therefore, even though high-dimensional data exhibit stronger concentration, the resulting shells still have large radii and non-negligible thickness.

\begin{figure*}[t]
    \centering
    \includegraphics[width=0.99\textwidth, height=0.41\textwidth]{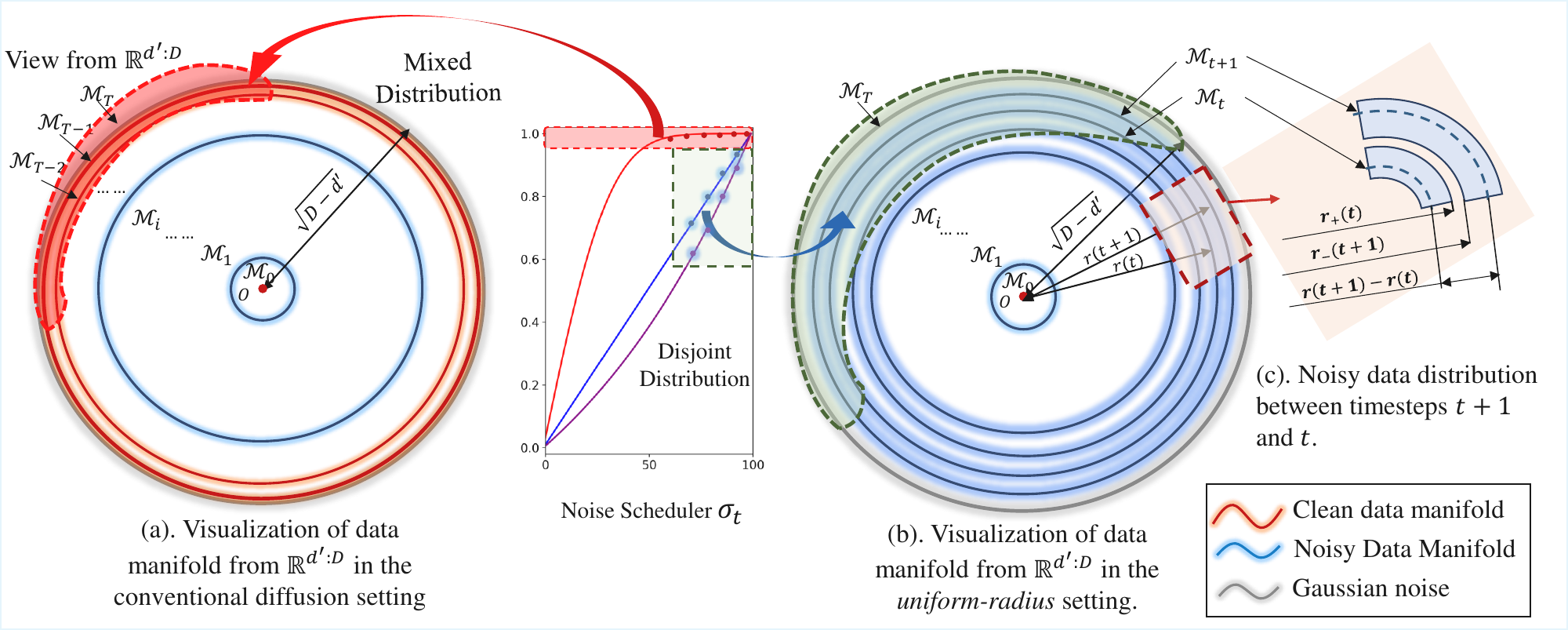}
    \caption{Distribution mixing under geometric view. (a) The conventional VP schedule compresses noisy manifolds near the high-noise end. (b) Uniform radius and late expansion radial schedules maintain more uniform spacing. (c) The thickness of each noisy shell determines whether adjacent timesteps overlap. The blue, red, and purple curves correspond to conventional, uniform-radial, and late-expansion VP noise schedules.}
    \label{fig:method}
    \vspace{-0.2in}
\end{figure*}

\begin{prop}\label{prop:distribution_mixing}
Let $\mathcal{A}_t^\varepsilon$ denote the high-probability shell around $\mathcal{M}_t$ given by ~\eqref{eq:A_t_varepsilon}.
If two adjacent shells overlap, namely $\mathcal{A}_t^\varepsilon \cap \mathcal{A}_{t+1}^\varepsilon \neq \emptyset$, then the corresponding noisy distributions are mixed.
A sufficient condition to avoid adjacent mixing is
\begin{equation}
    r_-(t+1) > r_+(t).
\label{eq:avoid_mixing_radius}
\end{equation}
Equivalently, 
\begin{equation}
    \Delta r(t)\defeq r(t+1) - r(t)>r(t)\Bigl(\sqrt{1+2\sqrt{\varepsilon}+2\varepsilon}-\sqrt{1-2\sqrt{\varepsilon}}\Bigr).
\label{eq:avoid_mixing}
\end{equation}
\end{prop}

Proposition~\ref{prop:distribution_mixing} gives the basic non-overlap criterion in terms of shell radii, and the proof is direct, as shown in Figure \ref{fig:method}.

For discrete schedules, it is useful to restate it directly in terms of adjacent noise amplitudes. Proposition~\ref{prop:distribution_mixing} directly implies the following corollary.

\begin{cor}
Fix $\varepsilon>0$ and define
\begin{equation}
    \rho_{\varepsilon}
    :=
    \sqrt{\frac{1+2\sqrt{\varepsilon}+2\varepsilon}{1-2\sqrt{\varepsilon}}}.
\end{equation}
If adjacent timesteps satisfy
\begin{equation}
    \sigma_{t+1} > \rho_{\varepsilon}\,\sigma_t,
    \label{eq:sigma_ratio_condition}
\end{equation}
then the corresponding shell regions in $\mathbb{R}^{D-d'}$ do not intersect.
Equivalently,
\begin{equation}
    \Delta \sigma_t := \sigma_{t+1}-\sigma_t > (\rho_{\varepsilon}-1)\sigma_t.
    \label{eq:sigma_gap_condition}
\end{equation}
\end{cor}

~\eqref{eq:sigma_gap_condition} gives the practical answer to the question “how far apart must $\sigma_t$ and $\sigma_{t+1}$ be?”
The required separation is not constant.
It scales linearly with the current noise amplitude $\sigma_t$.
Therefore, as $t$ increases and the shell gets thicker, adjacent timesteps must be spaced larger to avoid overlap.
This reveals precisely that the late-stage mixing is the dominant failure mode in time-unconditional DDIM under a conventional noise scheduler, that $\sigma_t$ is an increasing but $\frac{d\sigma_t}{dt}$ gradually decreases over time. Consequently, $\sigma_{t+1}-\sigma_t$ decrease with $t$ and approaches $0$ in the late stage. This directly contradicts ~\eqref{eq:sigma_gap_condition} for any $\epsilon>0$, which requires the separation between adjacent timesteps to grow proportionally with $\sigma_t$. Hence, as demonstrated in Figure \ref{fig:method}(a), the conventional schedule fails where disentanglement is most needed: when the shell becomes thicker, adjacent shells are placed increasingly closer in radius, inevitably causing late-stage manifold overlap and timestep mixing.

If ~\eqref{eq:sigma_gap_condition} fails, then there is a region where the noisy data distribution of two different timesteps both have non-negligible probability mass.
In that region, a noise predictor trained without $t$ cannot simultaneously represent the correct DDPM/DDIM targets for both timesteps.
The target necessarily collapses to a posterior average over timesteps, which is not the score field required by the reverse dynamics \cite{chen2023score}. By contrast, if adjacent shells are disjoint, then $p(t\mid \boldsymbol{x})$ is effectively concentrated, and the noisy sample itself identifies the timestep.
Time conditioning becomes redundant because manifold disentanglement has already encoded time into geometry.

\begin{figure*}[t]
    \centering
    \includegraphics[width=0.92\textwidth, height=0.82\textwidth]{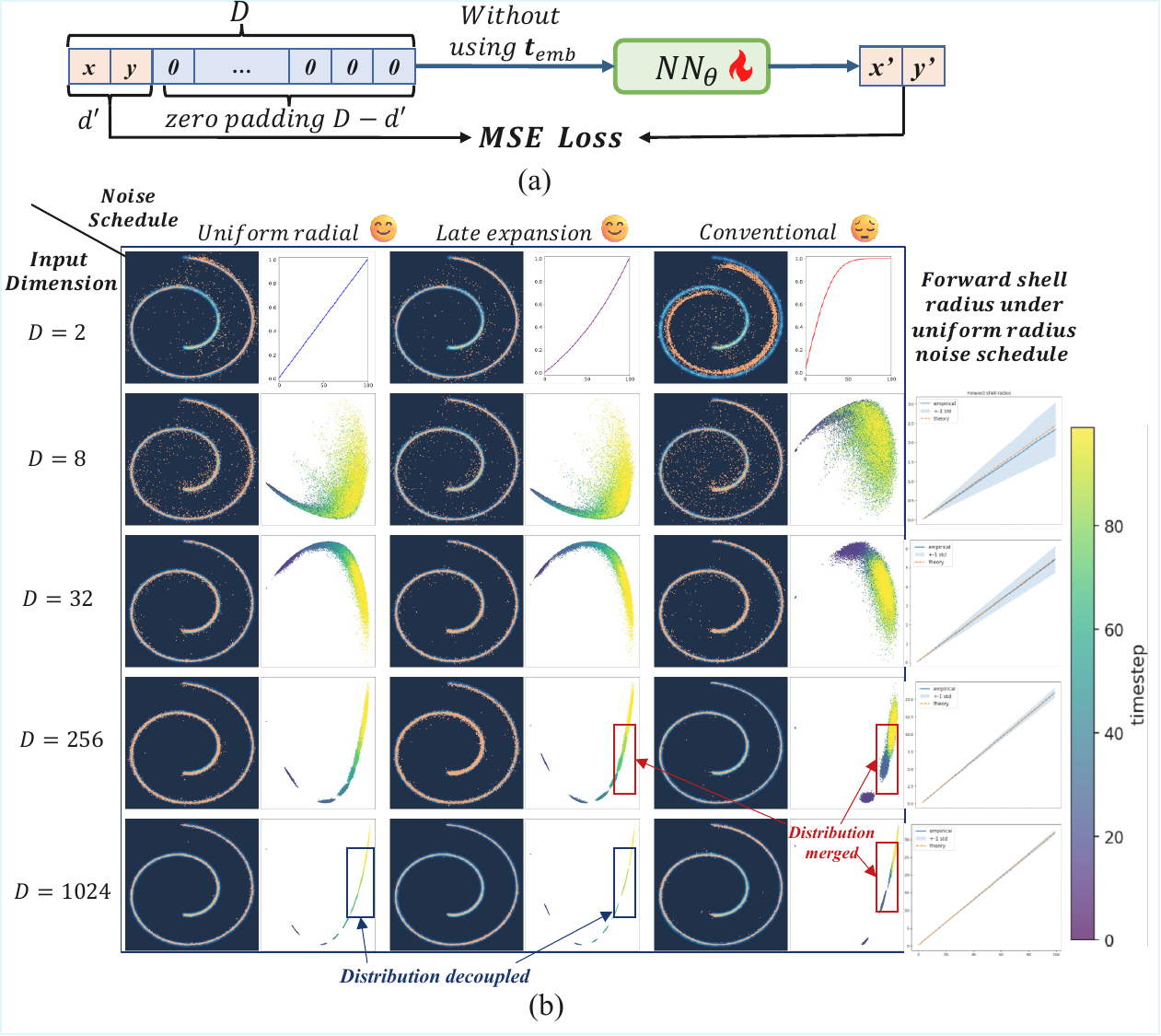}
    \caption{Swiss-roll toy experiments without timestep embedding. Rows vary the ambient dimension $D$, schedule blocks compare Uniform-radial VP, Late-expansion VP, and Conventional VP, and the rightmost column shows the corresponding shell-radius curves.}
    \label{fig:toydataset}
    \vspace{-0.1in}
\end{figure*}

\subsection{Manifold disjoint via noise schedule }

\begin{figure*}[t]
    \centering
    \includegraphics[width=0.84\textwidth, height=0.58\textwidth]{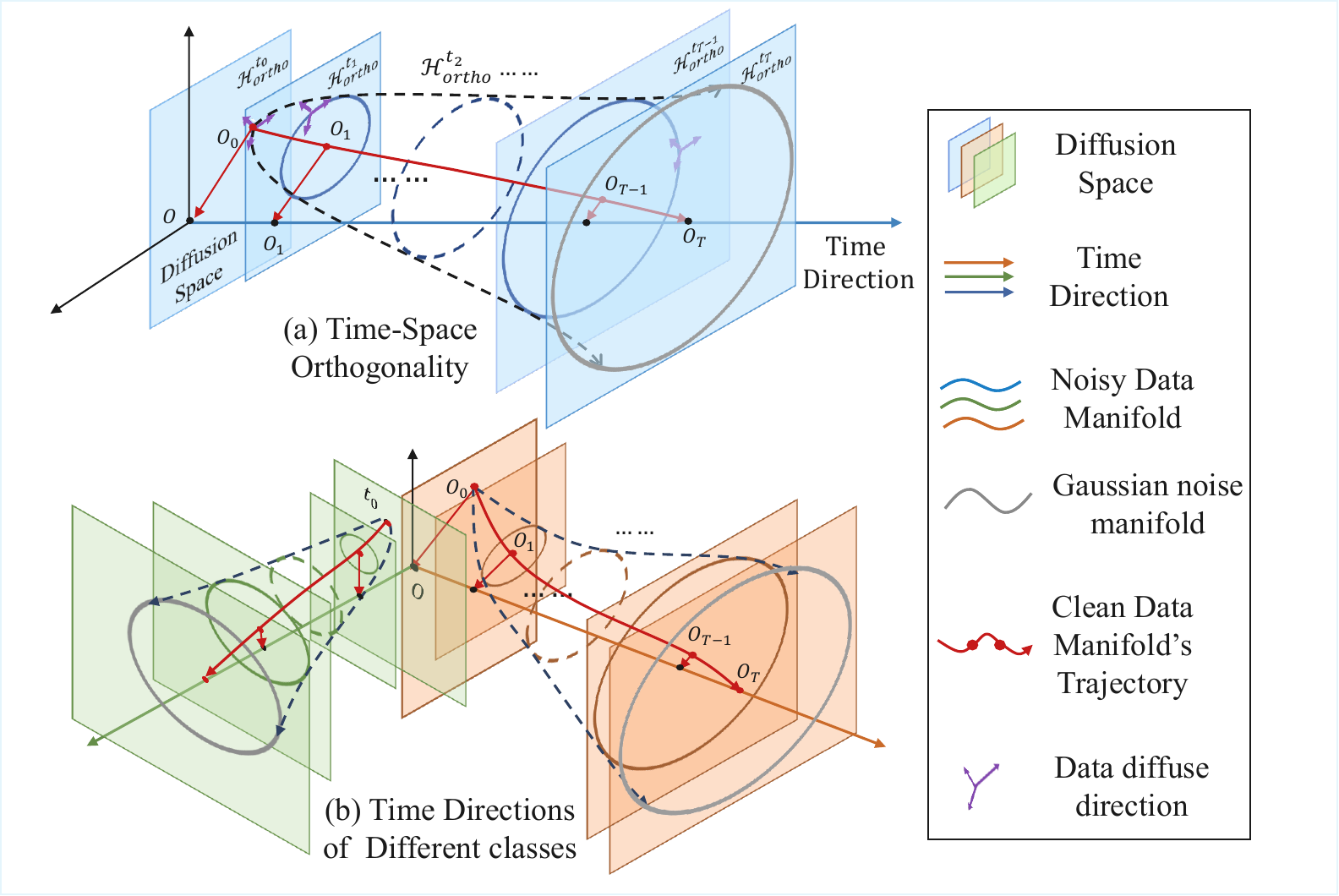}
    \caption{Orthogonal time-space disentanglement. (a) Time is encoded as a geometric direction $\boldsymbol{t}_{ortho}$, while diffusion takes place in the orthogonal subspace. (b) Assigning different time directions to different classes enables class-conditional generation without class or time embeddings.}
    \label{fig:ortho}
    \vskip -0.15in
\end{figure*}

The first way to realize time-unconditional diffusion is to redesign the forward path so that adjacent noisy shells remain separated for as long as possible.
Instead of prescribing $\beta_t$ first, we control the normal-space radius through $\sigma_t$ directly and then set $c_t=\sqrt{1-\sigma_t^2}$ in the unified forward form.
We compare four representative schedules:
\begin{equation}
    \begin{alignedat}{2}
        \text{Uniform-radial VP:}\quad \sigma_t &= t, &\qquad
        \text{Late-expansion VP:}\quad \sigma_t &= \frac{e^t-1}{e-1}, \quad t\in[0,1].
    \end{alignedat}
\end{equation}

For reference, we also consider the conventional VP schedule used in DDPM/DDIM and the optimal transport(OT) geodesic path used in flow matching.
The resulting parameters are summarized in Table~\ref{tab:schedule}.

\begin{table}[htbp]
\centering
\caption{Noise Parameters for Different Noise Scheduling Strategies}
\renewcommand{\arraystretch}{1.} 
\begin{tabular}{lcccc}
\toprule
\textbf{Noise Schedule} & $\beta_t$ & $\alpha_t$ & $c_t$ & $\sigma_t$ \\ \midrule
Conventional VP & $kt + b_0$ & $1 - kt - b_0$ & $\sqrt{\prod_{s=1}^{t} (1 - ks - b_0)}$ & $\sqrt{1 - \prod_{s=1}^{t} (1 - ks - b_0)}$ \\ 
Uniform-radial VP & $\frac{2t\Delta t - (\Delta t)^2}{1 - (t - \Delta t)^2}$ & $\frac{1 - t^2}{1 - (t - \Delta t)^2}$ & $\sqrt{1 - t^2}$ & $t$ \\ 
Late-expansion VP & derived from $\sigma_t$ & derived from $\sigma_t$ & $\sqrt{1 - \left(\frac{e^t-1}{e-1}\right)^2}$ & $\frac{e^t-1}{e-1}$ \\ 
OT geodesic path & - & - & $1 - t$ & $t$ \\ \bottomrule
\label{tab:schedule}
\end{tabular}
\end{table}

These schedules admit a clean geometric interpretation:
\begin{itemize}
    \item \textbf{Conventional VP schedule.} This is the usual DDPM/DDIM parameterization induced by a linear or cosine $\beta_t$ schedule. Here $\sigma_t=\sqrt{1-\bar{\alpha}_t}$ saturates near the terminal stage, so $\Delta\sigma_t$ shrinks exactly where the shell width is largest. This is the schedule most prone to manifold overlap.
    \item \textbf{Uniform-radial VP schedule.} This schedule makes the noisy-manifold radius grow linearly with time. It removes the terminal compression of the conventional VP schedule and substantially reduces overlap. Geometrically, it is the DDIM analogue that best matches the constant radial growth of OT-style paths.
    \item \textbf{Late-expansion VP schedule.} This schedule increases the radial gap more aggressively at large $t$. Since shell thickness grows with $\sigma_t$, this is the schedule designed specifically to combat the late-stage overlap predicted by ~\eqref{eq:sigma_gap_condition}.
    \item \textbf{OT geodesic path.} This is the flow-matching path~\cite{lipman2022flow,chewi2023log}. In the normal subspace $\mathbb{R}^{d':D}$ that is orthogonal to $\mathcal{M}_0$, it induces a uniform radial law, which is why it naturally exhibits much weaker manifold mixing.
\end{itemize}

\begin{prop}\label{prop:exp_schedule}
Let
\begin{equation}
    \sigma_t = \frac{e^{t}-1}{e-1},
    \qquad
    \kappa_{\varepsilon} := \sqrt{1+2\sqrt{\varepsilon}+2\varepsilon} - \sqrt{1-2\sqrt{\varepsilon}}.
\label{eq:exp_alpha}
\end{equation}
If $e^{\Delta t}-1 > \kappa_{\varepsilon}$, then the Late-expansion VP schedule satisfies ~\eqref{eq:avoid_mixing} for every adjacent pair of timesteps.
\end{prop}

The proposition formalizes the intuition above: once shell thickness grows with $t$, a schedule with constant radial increment can become marginal, whereas a late-expanding radial law can maintain separation all the way to the terminal stage.
Importantly, none of these schedule-level changes modify the denoiser architecture.
They only reorganize the geometry of noisy data, which is exactly where the time-conditioning problem comes from. Figure~\ref{fig:toydataset} provides an empirical illustration on Swiss-roll toy data without timestep input. Across rows, increasing the ambient dimension $D$ makes the shell structure more pronounced and helps the generated samples adhere more closely to the low-dimensional spiral support. This is consistent with the shell-concentration picture above, where the geometry of noisy samples becomes easier to organize around radius-controlled manifolds in higher-dimensional normal spaces.

\subsection{Disentangle time space via orthogonal time-space decomposition}

Schedule redesign improves separation, but it still depends on the particular choice of $\sigma_t$.
Our second solution is more structural: instead of asking the schedule alone to keep noisy shells apart, we directly separate the time direction from the diffusion direction in the full data space \cite{karczewski2026the}.
This allows us to preserve the original variance-preserving schedule while enforcing manifold disentanglement geometrically. Figure~\ref{fig:toydataset} also supports this motivation. Across columns, Uniform-radial VP and Late-expansion VP keep neighboring shells more separated than the Conventional VP schedule, especially near the high-noise regime where overlap is predicted by \eqref{eq:sigma_gap_condition}. This shows the value of schedule design, but it also highlights its limitations. If radial spacing inside one diffusion space is still insufficient, we can separate timesteps through a new geometric direction.

Under Assumption~\ref{assum1}, we choose a unit direction $\boldsymbol{t}_{ortho}$ orthogonal to $\mathcal{M}_0$ and set $\boldsymbol{v}=\delta\,\boldsymbol{t}_{ortho}$ with $\delta>0$. We then define the base diffusion hyperplane
\begin{equation}
    \mathcal{H}_{ortho}
    :=
    \{\boldsymbol{x}\in\mathbb{R}^D:\langle \boldsymbol{x}, \boldsymbol{t}_{ortho} \rangle = 0\},
\end{equation}
where $\langle \cdot, \cdot \rangle$ denotes the inner product. By construction, every clean sample satisfies $\langle \boldsymbol{x}_0, \boldsymbol{t}_{ortho} \rangle=0$, so
\begin{equation}
    \mathcal{M}_0 \subset \mathcal{H}_{ortho}.
\end{equation}
Under the linear-manifold assumption, this inclusion is exact: the clean manifold lies entirely inside the base diffusion hyperplane. That is important because it lets us use $\boldsymbol{t}_{ortho}$ only to separate timesteps, while keeping the clean signal inside the diffusion slice, exactly as illustrated in Figure~\ref{fig:ortho}(a). For each timestep, we define the shifted diffusion hyperplane
\begin{equation}
    \mathcal{H}_{ortho}^{t}
    :=
    \mathcal{H}_{ortho} + t\boldsymbol{v}
    =
    \{\boldsymbol{x}\in\mathbb{R}^D:\langle \boldsymbol{x}, \boldsymbol{t}_{ortho} \rangle = t\delta\}.
    \label{eq:ortho_hyperplane_t}
\end{equation}
Thus $\mathcal{H}_{ortho}^{0}=\mathcal{H}_{ortho}$ contains the clean manifold, and each later $\mathcal{H}_{ortho}^{t}$ is a parallel copy shifted along the time direction. The family $\{\mathcal{H}_{ortho}^{t}\}$ is the sequence of blue slices shown in Figure~\ref{fig:ortho}(a). Given Gaussian noise $\boldsymbol{z}\sim\mathcal{N}(\boldsymbol{0},\boldsymbol{I})$, we project it onto $\mathcal{H}_{ortho}$
\begin{equation}
    \boldsymbol{z}_{ortho}
    =
    \boldsymbol{z}
    -
    \langle \boldsymbol{z}, \boldsymbol{t}_{ortho} \rangle \boldsymbol{t}_{ortho}.
\end{equation}
With this decomposition, the modified VP forward process becomes
\begin{equation}
    \boldsymbol{x}_t = c_t\boldsymbol{x}_0 + t\boldsymbol{v} + \sigma_t\boldsymbol{z}_{ortho},
    \label{eq:forward_ortho}
\end{equation}
where $\boldsymbol{x}_0\in\mathcal{M}_0\subset\mathcal{H}_{ortho}$ and $\boldsymbol{z}_{ortho}\in\mathcal{H}_{ortho}$. Therefore,
\begin{equation}
    \langle \boldsymbol{x}_t, \boldsymbol{t}_{ortho} \rangle
    =
    c_t\langle \boldsymbol{x}_0, \boldsymbol{t}_{ortho} \rangle
    +
    t\langle \boldsymbol{v}, \boldsymbol{t}_{ortho} \rangle
    +
    \sigma_t\langle \boldsymbol{z}_{ortho}, \boldsymbol{t}_{ortho} \rangle
    =
    t\delta,
\end{equation}
which implies
\begin{equation}
    \boldsymbol{x}_t \in \mathcal{H}_{ortho}^{t}.
\end{equation}
This relation is the key geometric fact behind the construction. The signal term $c_t\boldsymbol{x}_0$ and the diffusion term $\sigma_t\boldsymbol{z}_{ortho}$ stay inside a single slice, while time acts only by translating that slice through $t\boldsymbol{v}$. In the language of Figure~\ref{fig:ortho}(a), diffusion happens inside each plane, whereas the planes themselves move along the time direction.

Because diffusion is restricted to $\mathcal{H}_{ortho}^{t}$ at timestep $t$, manifolds from different timesteps occupy different parallel slices of the full data space.
Even if the radial geometry within each slice still resembles a VP-style noisy shell, those shells now live in different hyperplanes rather than competing in one common ambient slice.
Hence, the model no longer needs to infer the timestep from a possibly overlapping radial structure alone. The timestep is already encoded by which shifted hyperplane $\mathcal{H}_{ortho}^{t}$ the sample belongs to.

This viewpoint also explains why orthogonal time-space disentanglement is more general than schedule tuning.
Schedule tuning tries to keep shells apart inside the same normal space; orthogonal decomposition instead creates a new direction along which manifolds are separated by construction.
The first approach reduces overlap, while the second enforces disentanglement. Figure~\ref{fig:ortho}(a) makes this distinction visual. The purple arrows indicate diffusion directions inside each slice, while the sequence of planes indexed by $t_0,t_1,\dots,t_T$ records the translation of those slices along the time axis.
Similarly, \cite{brown2023verifying} argues that real-world data resides on a disjoint union of manifolds with varying intrinsic dimensions. This perspective also extends naturally to class-conditional generation.
By assigning each class its own direction $\boldsymbol{t}_{ortho}^{\,i}$, as sketched in Figure~\ref{fig:ortho}(b), we separate both timesteps and classes into disjoint geometric slices, making class-conditioned sampling possible without introducing class embeddings into the denoiser~\cite{ho2022classifier, dhariwal2021diffusion}.

The practical training and sampling procedures are summarized below. The training loss is the same noise-prediction objective used in DDPM, while the sampling update follows the deterministic DDIM rule.

\begin{minipage}{0.48\textwidth}
\begin{algorithm}[H]
\caption{Training model $\boldsymbol{z}_{\boldsymbol{\theta}}$ with time direction $\boldsymbol{t}_{ortho}$}
\label{alg:E}
\begin{algorithmic}[1]
\STATE Choose a unit direction $\boldsymbol{t}_{ortho}$ and set $\boldsymbol{v}=\delta\boldsymbol{t}_{ortho}$
\STATE \textbf{repeat}
\STATE \textbf{Input:} timestep $t$, training data $\boldsymbol{x}_0$
\STATE Sample $\boldsymbol{z} \sim \mathcal{N}(\boldsymbol{0}, \boldsymbol{I})$ and compute $\boldsymbol{z}_{ortho}=\boldsymbol{z}-\langle \boldsymbol{z},\boldsymbol{t}_{ortho}\rangle\boldsymbol{t}_{ortho}$
\STATE $\boldsymbol{x}_t\gets c_t\boldsymbol{x}_0+t\cdot\boldsymbol{v}+\sigma_t\boldsymbol{z}_{ortho}$
\STATE Update model with $\nabla_{\boldsymbol{\theta}}\left\| \boldsymbol{z}_{ortho}-\boldsymbol{z}_{\boldsymbol{\theta}}(\boldsymbol{x}_t) \right\|^2$
\STATE \textbf{until} converged
\end{algorithmic}
\end{algorithm}
\end{minipage}
\hfill
\begin{minipage}{0.48\textwidth}
\begin{algorithm}[H]
\caption{DDIM sampling with time direction $\boldsymbol{t}_{ortho}$}
\label{alg:F}
\begin{algorithmic}[1]
\STATE Sample $\boldsymbol{x}_T \sim \mathcal{N}(\boldsymbol{0}, \boldsymbol{I})$ and project it to $\mathcal{H}_{ortho}$
\STATE $\boldsymbol{x}_T\gets \boldsymbol{x}_T+T\cdot\boldsymbol{v}$
\FOR{$t$ from $T$ to $1$}
\STATE $\hat{\boldsymbol{\epsilon}}_t \gets \boldsymbol{z}_{\boldsymbol{\theta}}(\boldsymbol{x}_t)$
\STATE $\hat{\boldsymbol{x}}_0 \gets \dfrac{\boldsymbol{x}_t-t\cdot\boldsymbol{v}-\sigma_t\hat{\boldsymbol{\epsilon}}_t}{c_t}$
\STATE $\boldsymbol{x}_{t-1} \gets c_{t-1}\hat{\boldsymbol{x}}_0+(t-1)\cdot\boldsymbol{v}+\sigma_{t-1}\hat{\boldsymbol{\epsilon}}_t$
\ENDFOR
\STATE \textbf{return} $\boldsymbol{x}_0$
\end{algorithmic}
\end{algorithm}
\end{minipage}

\section{Results}
\label{sec:results}

\begin{figure*}[t]
    \centering
    \includegraphics[width=0.96\textwidth, height=0.42\textwidth]{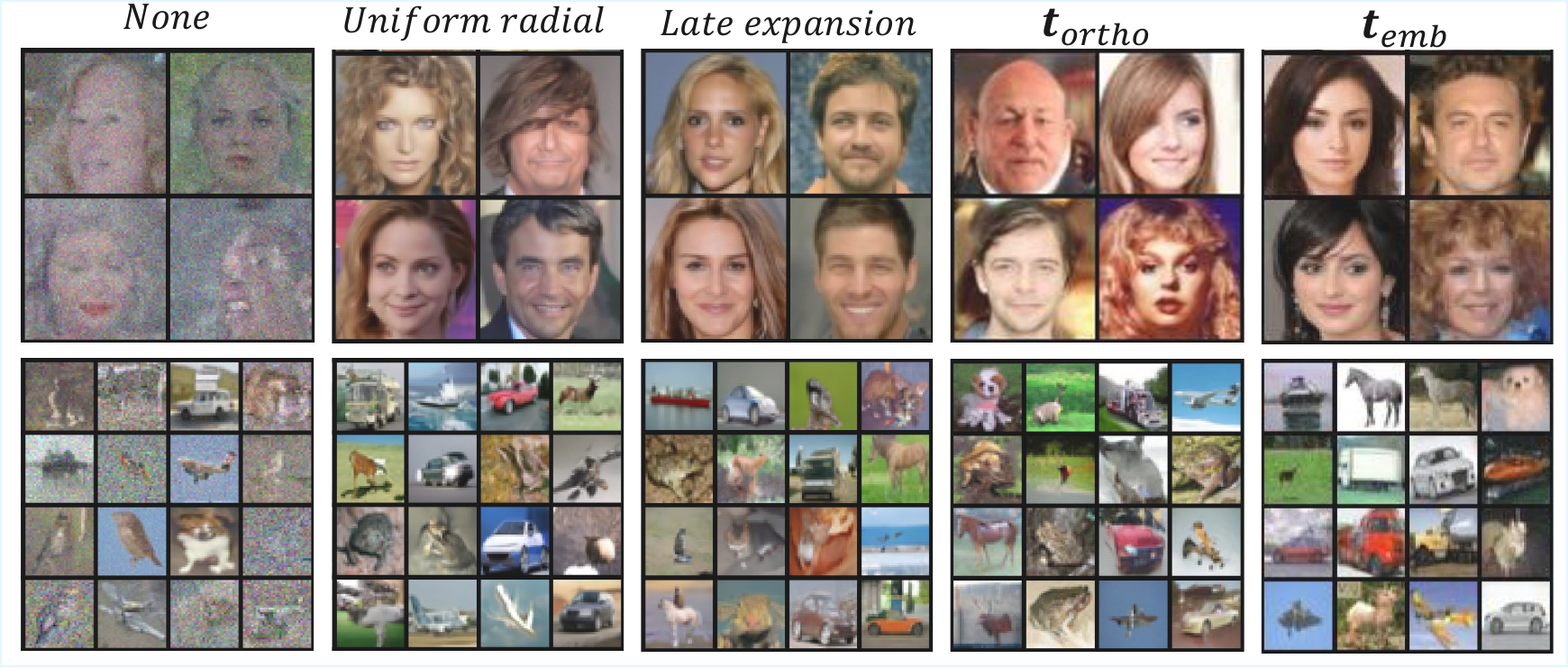}
    \caption{Qualitative DDIM comparison on the small-image benchmarks. Columns from left to right correspond to None, Uniform-radial, Late-expansion, $\boldsymbol{t}_{ortho}$, and $t_{emb}$. The top row shows CelebA samples, and the bottom row shows class-conditional CIFAR10 samples. The None baseline exhibits saturation, under-saturation, and residual noise, whereas all geometry-based decoupling strategies restore recognizable samples and the orthogonal construction is visually closest to the explicit time-conditioned baseline.}
    \label{fig:celeba_conditioning}
    \vskip -0.2in
\end{figure*}

We validate the geometric analysis by training diffusion models from scratch on small datasets: CIFAR10 \cite{Krizhevsky2009LearningML} and CelebA \cite{liu2015deep}, and by evaluating large datasets-ImageNet \cite{deng2009imagenet} generation.
Across all experiments, the main question is whether improved geometric separation of noisy manifolds can replace explicit time conditioning.
Unless otherwise noted, all quantitative results and all visual samples in this section are obtained with DDIM sampling~\cite{song2020denoising}.
This choice is deliberate: compared with stochastic samplers, DDIM is more sensitive to local geometric ambiguity between neighboring noisy manifolds, so it provides a particularly sharp test of the claims in Section~\ref{sec:time_disentangled_diffusion}~\cite{chung2022improving, he2023manifold, yao2025manifold, zhan2026understanding}.
If adjacent manifolds remain mixed, the deterministic reverse trajectory quickly accumulates errors; if they are separated, the sample geometry itself can reveal the current manifold and make explicit time tokens unnecessary.

\subsection{Rescheduling alleviates time-free DDIM failure}

We first test the schedule-level claim from Section~\ref{sec:time_disentangled_diffusion} that if DDIM fails because neighboring noisy manifolds overlap under the standard VP trajectory used by the None setting, then reducing this overlap should improve time-unconditional generation even when the denoiser never receives an explicit time embedding. This experiment also revisits, from a geometric viewpoint, the time-unconditional DDIM failure reported by \cite{sun2025noise}. This is exactly what we observe.
Under the None setting, DDIM collapses on all three small-image benchmarks, producing the oversaturated or noisy outputs shown in Figure~\ref{fig:celeba_conditioning}.
Quantitatively, the FID rises to $71.25$ on class-conditional CIFAR10, $66.03$ on CIFAR10-unconditional, and $120.71$ on CelebA, far from the explicit time-embedding baseline in Table~\ref{tab:time_conditioning_results}.

\begin{wrapfigure}{r}{0.5\linewidth}
    \centering  
    \vspace{-0.1in}
    \includegraphics[width=0.9\linewidth, height=0.85\linewidth]{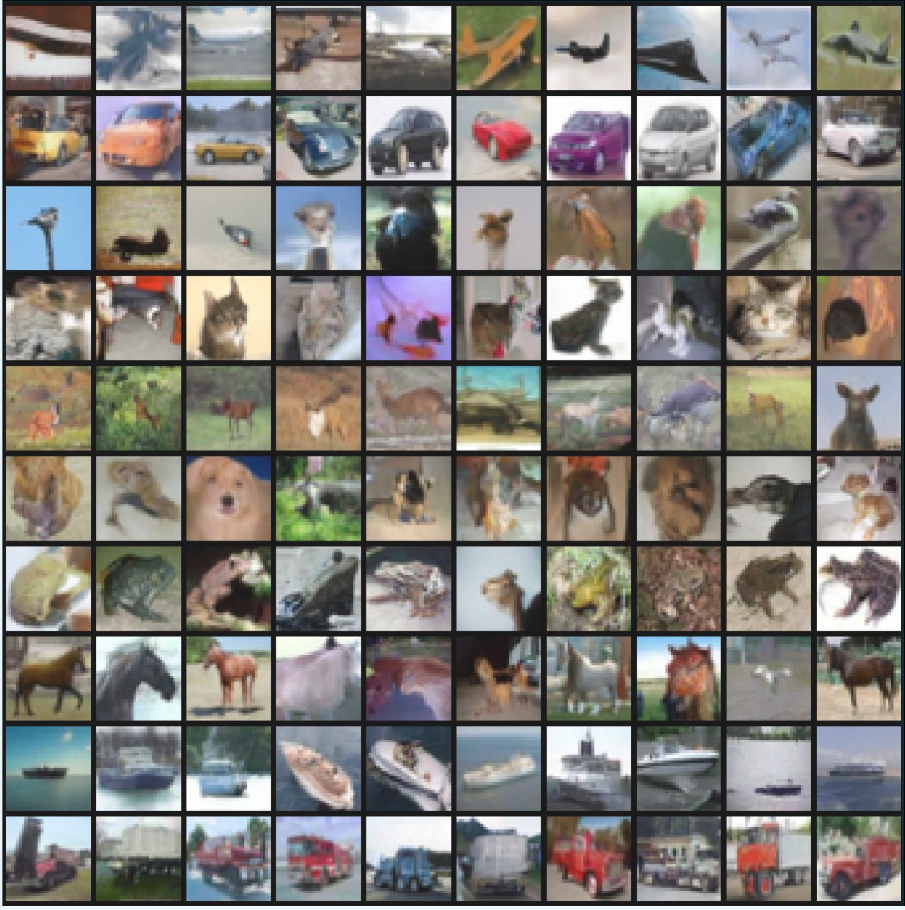}  
    \caption{Qualitative class-conditional CIFAR10 results under $\boldsymbol{t}_{emb}$  in Figure \ref{fig:ortho}. Each row corresponds to one CIFAR10 class shift direction, providing a visualization of the class-specific time-space disentanglement.}
    \label{fig:merge_class}
    \vspace{-0.2in}
\end{wrapfigure}

Once the VP forward process in \eqref{eq:forward_process} is rescheduled to enlarge the neighborhood manifold spacing, the performance of the time-unconditional model improves immediately.
As shown in Table \ref{tab:time_conditioning_results}, the uniform-radial scheduler lowers the FID to $7.59$, $10.59$, and $15.43$ on the same three datasets, while the late-expansion scheduler yields $7.46$, $13.66$, and $18.92$.
These gains should not be explained by architectural effects, since the denoiser itself is unchanged; the only modification is the geometry of the noisy trajectory.
This directly supports the Proposition~\ref{prop:distribution_mixing} that DDIM is not primarily failing because the model is too weak, but because the standard VP process \ref{eq:forward_process} setting compresses neighboring manifolds exactly where shell thickness is largest, making the reverse target ambiguous without providing explicit time information in the time-unconditional setting.

\begin{figure*}[t]
    \centering
    \includegraphics[width=0.92\textwidth, height=0.4\textwidth]{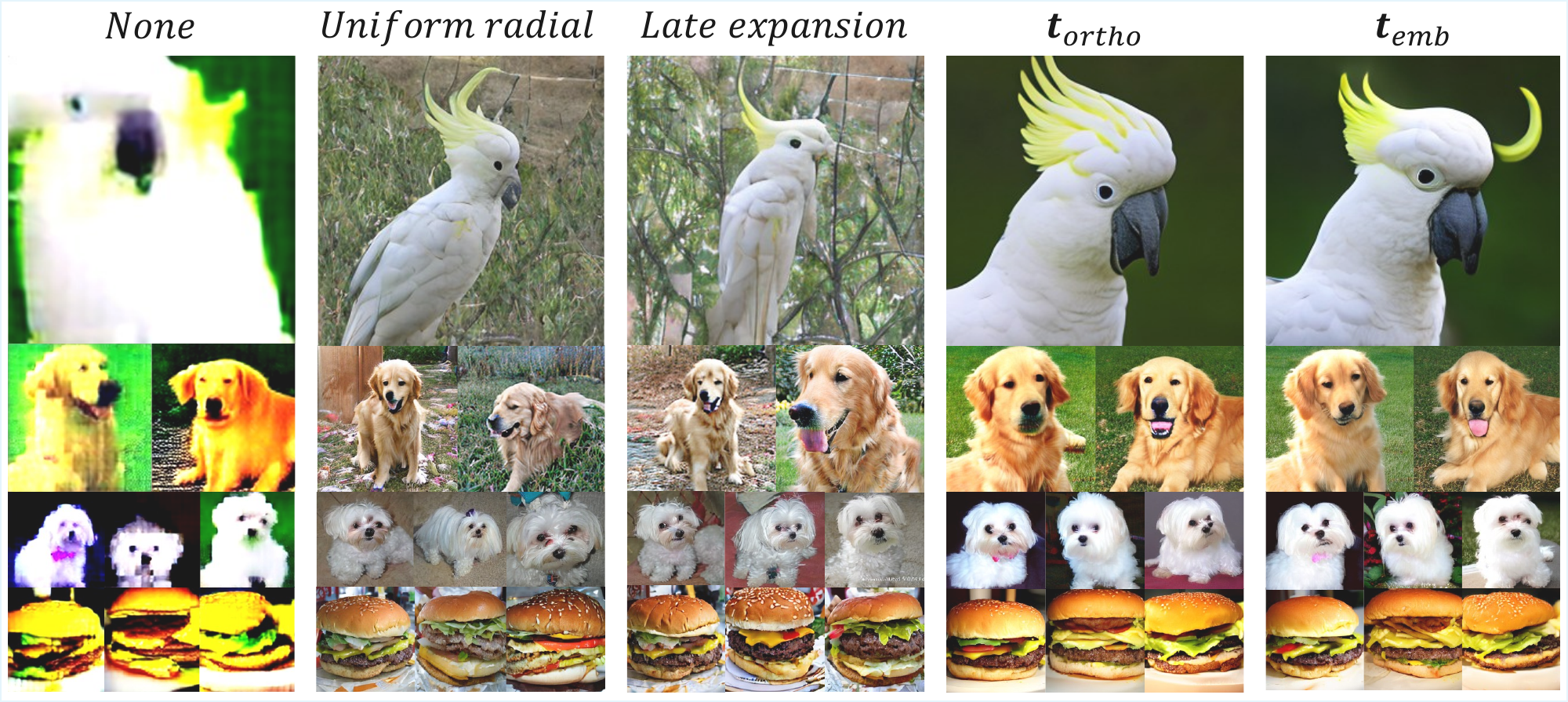}
    \caption{Qualitative ImageNet result with DiT-B/2~\cite{peebles2023scalable} at $256 \times 256$, DDIM sampling~\cite{song2020denoising}, and classifier-free guidance scale $5.5$~\cite{ho2022classifier, dhariwal2021diffusion}. Columns from left to right are None, Uniform-radial, Late-expansion, $\boldsymbol{t}_{ortho}$, and $t_{emb}$. Each row shows representative ImageNet classes. The None baseline collapses, while schedule-level manifold disjoint and orthogonal time-space separation recover clean samples without explicitly feeding the timestep into the denoiser.}
    \label{fig:ImageNet_comparison}
    \vskip -0.1in
\end{figure*}

The same conclusion becomes even sharper in the ImageNet with diffusion-transformer experiments.
At fixed DiT-B/2 capacity~\cite{peebles2023scalable} and classifier-free guidance scale $5.5$~\cite{ho2022classifier, dhariwal2021diffusion}, the \textit{None} baseline still fails badly, with high FID $35.73$, sFID $35.47$ and extremely low Inception Score $90.97$ \cite{salimans2016improved}.
However, the two schedule-only time-unconditional variants again recover strong performance: Uniform-radial VP reaches FID $9.33$ and sFID $9.87$, and Late-expansion VP reaches FID $12.46$ and sFID $13.41$.
In fact, under this training and sampling budget, Uniform-radial VP even surpasses the explicit time-embedding baseline on FID and sFID, while remaining competitive in precision and Inception Score.
This result is important for the overall thesis of the paper as it shows that once a sufficiently expressive diffusion transformer is paired with a geometry that separates neighboring noisy manifolds, explicit timestep embeddings are no longer strictly necessary for high-quality DDIM sampling.
\begin{table}[t]
  \caption{Quantitative evaluation on the Celeba dataset (unconditional) and Cifar10 dataset (conditional) of the generated image using a different time conditioning strategy.}
  \label{tab:time_conditioning_results}
  \centering
  \resizebox{\columnwidth}{!}{
  \begin{tabular}{lccccccccc}
    \toprule
    \multirow{2}{*}{\textbf{Time Conditioning}} 
      & \multicolumn{3}{c}{\textbf{CIFAR10}} 
      & \multicolumn{3}{c}{\textbf{CIFAR-unconditional}}
      & \multicolumn{3}{c}{\textbf{Celeba}} \\
    \cmidrule(lr){2-4} \cmidrule(lr){5-7} \cmidrule(lr){8-10}
      & FID$\downarrow$ 
      & Precision$\uparrow$ & Recall$\uparrow$
      & FID$\downarrow$
      & Precision$\uparrow$ & Recall$\uparrow$
      & FID$\downarrow$ 
      & Precision$\uparrow$ & Recall$\uparrow$ \\
    \midrule
    
    time embedding        & 4.58 & 0.729    & 0.522    & 5.04 & 0.728    & 0.521    & 9.99    & 0.707    & 0.414 \\  
    
    \rowcolor{lightgreen}
    Uniform-radial VP        & 7.59   & 0.674 & 0.495 & 10.59   & 0.682 & 0.478 & 15.43 & 0.736 & 0.271          \\

    \rowcolor{lightgreen}
    Late-expansion VP           & 7.46   & 0.689 & 0.462  & 13.66   & 0.689 &  0.450 & 18.92 & 0.77 & 0.195          \\

    \rowcolor{lightgreen}
    $\boldsymbol{t}_{ortho}$  &  6.57  & 0.713  & 0.507 & 7.74  & 0.730  & 0.476 & 11.38 & 0.676  & 0.483          \\

    \rowcolor{pink}
    None & 71.25      & 0.4115    & 0.555     & 66.03      & 0.424    & 0.565     & 120.71 & 0.106  & 0.604          \\
    \bottomrule
  \end{tabular}
  }
\end{table}

Taken together, the small-image and ImageNet results validate the main theoretical prediction of Section~\ref{sec:time_disentangled_diffusion}: in a manifold-sensitive sampler such as DDIM, explicitly feeding $t$ is not the only way to identify the current denoising manifold.
If the forward process is redesigned so that adjacent noisy manifolds are better separated, the model can infer the relevant time information from geometry alone.

\subsection{Orthogonal time-space disentanglement directly validates complete separation}

We next test the stronger construction from Section~\ref{sec:time_disentangled_diffusion}, where timesteps are separated by translation along an approximately orthogonal direction $\boldsymbol{t}_{ortho}$ rather than only by changing their radial spacing~\cite{esser2024scaling}.
This experiment is the most direct implementation of the theory because it enforces complete separation in ambient space, where diffusion happens inside the subspace, while timestep identity is carried by displacement along another.

The results match this prediction closely.
In Table~\ref{tab:time_conditioning_results}, the orthogonal variant reduces the FID from $71.25$ to $6.57$ on class-conditional CIFAR10, from $66.03$ to $7.74$ on CIFAR10-unconditional, and from $120.71$ to $11.38$ on CelebA.
These numbers are much closer to the explicit time-embedding baseline than to the no-time baseline, and the corresponding samples in Figure~\ref{fig:celeba_conditioning} are visibly cleaner.
This is precisely the qualitative behavior expected from the theory. Once different timesteps occupy parallel slices, the reverse model no longer has to resolve ambiguous overlaps inside a single noisy shell. The class-conditional CIFAR10 result is also conceptually important.
Here, we do not provide time or class information, but use class-specific time directions so that class label and timestep are both encoded geometrically.
According to Figure \ref{fig:merge_class}, the resulting DDIM samples remain class-consistent even though the method does not rely on an explicit timestep or class embedding, showing that the same disentanglement mechanism may be generalized from time conditioning to class conditioning~\cite{ho2022classifier, dhariwal2021diffusion}, the quantitative result in table \ref{tab:time_conditioning_results} is also valid for our method.

On ImageNet, the orthogonal variant remains competitive with the explicit time-embedding baseline, reaching FID $17.52$ versus $18.07$ while closely matching its precision and recall in Table~\ref{tab:imagenet_results}.
It does not produce the best FID among the ImageNet variants, but it is still the clearest validation of the complete time-space separation mechanism because it leaves the original VP-style schedule intact and solves the ambiguity by changing geometry alone.
In other words, the schedule-based variants show that weakening overlap is already sufficient, while $\boldsymbol{t}_{ortho}$ shows that removing overlap by construction is sufficient as well.

\begin{figure}[t]
    \centering
    \includegraphics[width=0.9\textwidth, height=0.37\textwidth]{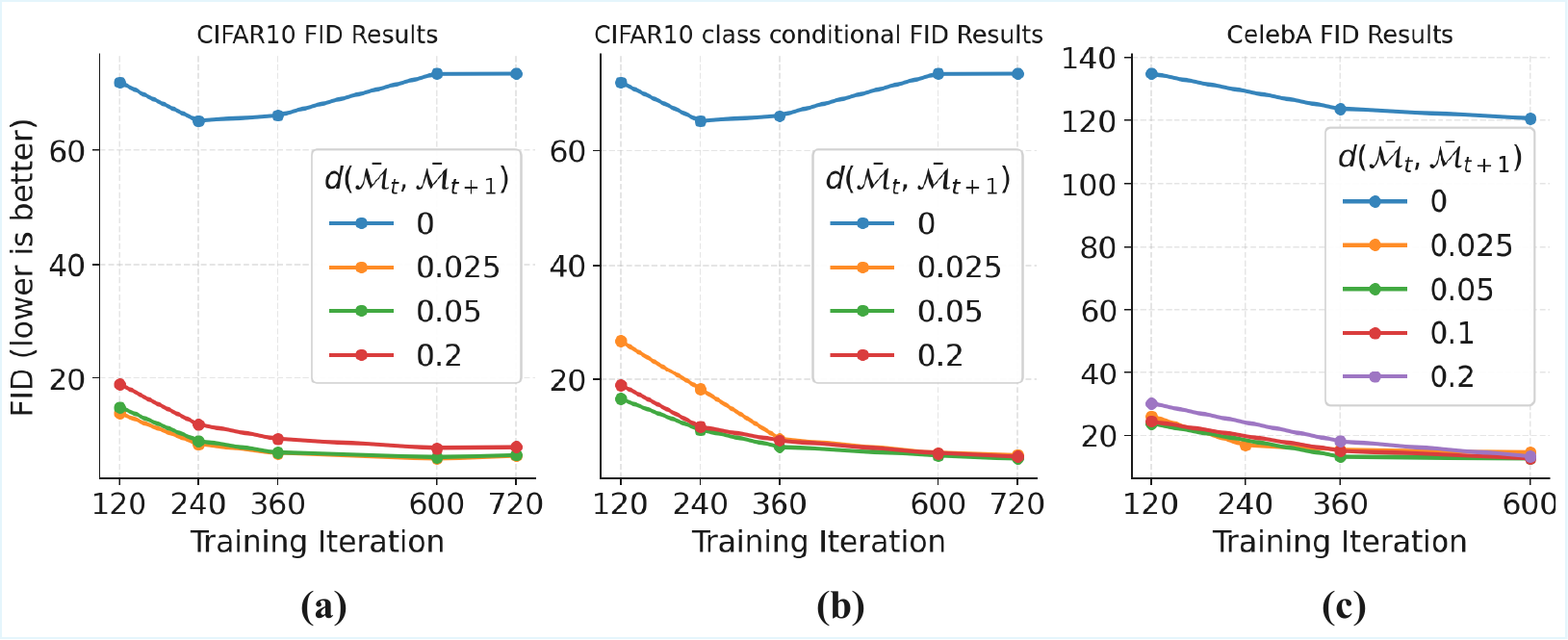}
    \caption{Ablation on the manifold spacing parameter $d(\mathcal{M}_t, \mathcal{M}_t)$ under DDIM sampling method. Subfigure (a) shows CIFAR10-unconditional, (b) shows class-conditional CIFAR10, and (c) shows the CelebA generation result. Each curve corresponds to a different spacing step length, i.e., the distance between adjacent manifolds induced by the orthogonal time direction. A very small positive spacing already removes most of the failure of the zero-spacing baseline, while excessively large spacing eventually degrades performance.}
    \label{fig:cifar_distance_ablation}
\end{figure}
\begin{table}[t]
    \centering
    \caption{Quantitative comparison on ImageNet under different time-conditioning strategies.}
    \begin{tabular}{lccccc}
    \toprule
    Time-Conditioning & FID$\downarrow$ & sFID$\downarrow$ & IS$\uparrow$ & Precision$\uparrow$  & Recall$\uparrow$  \\
    \midrule
    
    time embedding               & 18.07 & 20.45 & 308.31 & 0.932 & 0.11  \\
    \rowcolor{lightgreen}
    Uniform-radial VP     & 9.33 & 9.87 & 248.5 & 0.895 & 0.236 \\
    \rowcolor{lightgreen}
    Late-expansion VP     & 12.46 & 13.41 & 294.8 & 0.921 & 0.186 \\
    \rowcolor{lightgreen}
    $\boldsymbol{t}_{ortho}$     & 17.52 & 18.48 & 305.18 & 0.933 & 0.113 \\
    \rowcolor{pink}
    None                         & 35.73 & 35.47 & 90.97 & 0.432 & 0.323 \\
    \bottomrule
    \end{tabular}
    \label{tab:imagenet_results}

\end{table}

\subsection{Ablation on manifold spacing and ambient dimension}

Figure~\ref{fig:cifar_distance_ablation} studies the step-length parameter used by the orthogonal method to control the distance between adjacent manifolds.
The main empirical result demonstrates that moving from zero separation to a very small positive separation already yields a dramatic improvement.
For both CIFAR10 settings and for CelebA, spacing step lengths such as $0.025$ or $0.05$ eliminate most of the catastrophic behavior of the zero-separation baseline.
This is exactly what the shell-overlap theory predicts.
Once the gap between neighboring manifolds becomes slightly larger than the effective shell thickness, DDIM can reliably identify the current manifold and the major failure mode disappears.

Increasing the spacing further still helps in some settings, but the gain is much smaller and the optimum becomes dataset-dependent.
More importantly, the curves also show that making adjacent manifolds too far apart can hurt generation quality.
This is again geometrically sensible: once the manifolds are already separated enough to avoid mixing, further enlarging the gap makes the reverse transition less local and forces the denoiser to bridge unnecessarily large distances between consecutive timesteps.
The ablation therefore indicates a threshold rather than a monotonic increasing-some separation is essential, a small amount is usually sufficient, and overly aggressive separation is counterproductive.

We also test whether the input-dimension effect from the toy experiments remains visible on real images by padding CIFAR10 from $32 \times 32 \times 3$ to $64 \times 64 \times 3$.
The corresponding numbers are reported in Table~\ref{tab:input_dimension_ablation}.
Unlike the toy setting, the effect on real images is relatively small and not strictly monotonic.
This is consistent with the linear-manifold interpretation used throughout the theory section, which natural images already live on a low-dimensional manifold inside a very high-dimensional ambient space, so additional zero-padding changes the nominal ambient dimension much more than it changes the intrinsic geometry of the data distribution.
As a result, changing the schedule or explicitly separating neighboring manifolds has a much stronger effect than simply padding the image space.

\begin{table}[t]
  \caption{Ablation study on the effect of input dimension.}
  \label{tab:input_dimension_ablation}
  \centering
  \resizebox{0.8\columnwidth}{!}{
  \begin{tabular}{lcccccc}
    \toprule
    \multirow{2}{*}{\textbf{Time Conditioning}} 
      & \multicolumn{3}{c}{\textbf{original input $D=32\times32\times3$}} 
      & \multicolumn{3}{c}{\textbf{padding to $D=64\times64\times3$}}  \\
    \cmidrule(lr){2-4} \cmidrule(lr){5-7}
      & FID$\downarrow$ 
      & Precision$\uparrow$ & Recall$\uparrow$
      & FID$\downarrow$ 
      & Precision$\uparrow$ & Recall$\uparrow$ \\
    \midrule

    Uniform-radial VP        &10.59   &0.682 &0.386 &12.63   & 0.756 & 0.386          \\

    Late-expansion VP           &13.66   & 0.689 &0.450  & 12.74 & 0.753 & 0.362          \\

    \bottomrule
  \end{tabular}
  }
\end{table}

Overall, the ablations reinforce the main claim of this paper.
The relevant variable for time-unconditional DDIM is not whether the model receives a symbolic timestep token, but whether the noisy manifolds are separated enough for the sample geometry to identify the current denoising state.

\section{Discussion}
\label{sec:discussion}

Our results suggest that explicit time conditioning is not a fundamental requirement for diffusion-based generation.
What matters is whether the forward process organizes noisy samples from different timesteps into distinguishable geometric regions.
When the noisy manifolds overlap, a time-unconditioned denoiser receives conflicting supervision, and the generation quality deteriorates.
When they are disentangled, the current noise level can be inferred directly from sample geometry.

This viewpoint clarifies the practical gap between standard DDIM and flow matching, and it also motivates a broader design principle: conditioning information can often be removed from the network if it is instead encoded into the geometry of the data trajectory.
In our case, this leads to two complementary strategies.
Rescheduling improves the spacing between noisy manifolds while preserving the usual model architecture, and orthogonal time-space decomposition removes timestep ambiguity more fundamentally by embedding time directly into the ambient space.

Beyond timestep conditioning, the same construction extends naturally to class-conditioned generation.
This indicates that some conditioning mechanisms in diffusion models may be replaced by geometric disentanglement rather than additional embedding modules.
More broadly, our analysis offers a manifold-based perspective on why high-dimensional diffusion can remain effective despite the apparent difficulty of modeling complex multimodal distributions.

An immediate next step is to connect the theory more tightly to implementation details, including how to estimate the orthogonal time directions robustly, how schedule design interacts with resolution and intrinsic dimension, and how the same idea scales to larger conditional generation settings.
We leave those practical details and formal derivations to the following section.

\section{Details}
\label{sec:details}

This section collects technical derivations and implementation details that are intentionally separated from the main narrative.
We keep it concise and focus on the derivations and implementation details needed to reproduce the reported experiments.

\subsection{Background formulas for score-based diffusion and flow matching}

For completeness, we collect the standard formulas used by the background discussion in Section~\ref{sec:main}.
Diffusion models can be expressed by using the stochastic differential equations (SDE)~\cite{oksendal2013stochastic,le2016brownian}. For a dataset in $\mathbb{R}^{D}$, the continuous-time forward diffusion process can be written as using It\^o SDE~\cite{song2019generative, song2020denoising}
\begin{equation}
d\boldsymbol{x}
=
\boldsymbol{f}(\boldsymbol{x}, t)\, dt
+
\boldsymbol{g}(t)\, d\boldsymbol{w},
\label{eq:ddpmsde}
\end{equation}
where $\boldsymbol{w}$ is standard Brownian motion.
Under the variance-preserving parameterization of DDPM/DDIM~\cite{ho2020denoising, song2020denoising}, the closed-form forward marginal is written as \eqref{eq:forward_process}. The reverse-time dynamics depend on the time-indexed score field and can be expressed as~\cite{ANDERSON1982313,song2019generative, song2020denoising}
\begin{equation}
d\boldsymbol{x}
=
\left[
\boldsymbol{f}(\boldsymbol{x},t)
-
\boldsymbol{g}(t)^2\nabla_{\boldsymbol{x}}\log p(\boldsymbol{x}|t)
\right]dt
+
\boldsymbol{g}(t)\,d\clo{\boldsymbol{w}},
\label{eq1}
\end{equation}
and the denoiser is commonly trained by denoising score matching~\cite{vincent2011connection,song2020sliced}:
\begin{equation}
\mathcal{L}(\boldsymbol{\theta})
=
\mathbb{E}_{t,\boldsymbol{x}_0,\boldsymbol{x}_t}
\left[
\left\|
\boldsymbol{s}_\theta(\boldsymbol{x}_t,t)
-
\nabla_{\boldsymbol{x}_t}\log p(\boldsymbol{x}_t|t)
\right\|_2^2
\right].
\label{eq3}
\end{equation}

The marginal density evolves according to the Fokker--Planck equation~\cite{pavliotis2014stochastic,song2020score}
\begin{equation}
\frac{\partial p(\boldsymbol{x}|t)}{\partial t}
=
-\nabla_{\boldsymbol{x}}\cdot
\left[
\left(
\boldsymbol{f}(\boldsymbol{x},t)
-
\frac{1}{2}\boldsymbol{g}(t)^2\nabla_{\boldsymbol{x}}\log p(\boldsymbol{x}|t)
\right)
p(\boldsymbol{x}|t)
\right],
\label{eq:fp}
\end{equation}
which yields the probability-flow ODE~\cite{song2020score,song2020denoising}
\begin{equation}
\begin{aligned}
d\boldsymbol{x}
&=
\left[
\boldsymbol{f}(\boldsymbol{x},t)
-
\frac{1}{2}\boldsymbol{g}(t)^2\nabla_{\boldsymbol{x}}\log p(\boldsymbol{x}|t)
\right]dt \\
&= \boldsymbol{v}(\boldsymbol{x},t)\,dt.
\end{aligned}
\label{eq_pfode}
\end{equation}

Flow matching instead specifies a probability path directly.
A standard choice is the linear interpolation between data and noise~\cite{lipman2022flow,liu2023flow,lipman2024flow}:
\begin{equation}
\boldsymbol{x}_t
=
t\,\boldsymbol{x}_0
+
(1-t)\,\boldsymbol{z},
\qquad t\in[0,1],
\label{linear_interpolant}
\end{equation}
with conditional vector field
\begin{equation}
\boldsymbol{v}^\star(\boldsymbol{x}_t,t)
=
\frac{\boldsymbol{x}_0-\boldsymbol{x}_t}{1-t},
\end{equation}
and training objective
\begin{equation}
\mathcal{L}(\boldsymbol{\theta})
=
\mathbb{E}_{\boldsymbol{x}_0,\boldsymbol{z},t\in[0,1]}
\left[
\left\|
\boldsymbol{v}_\theta(\boldsymbol{x}_t,t)
-
\frac{\boldsymbol{x}_0-\boldsymbol{x}_t}{1-t}
\right\|_2^2
\right].
\label{eq:cfm}
\end{equation}

Both DDPM/DDIM and flow matching can therefore be summarized by the unified forward notation
\begin{equation}
\boldsymbol{x}_t
=
\boldsymbol{c}(t)\,\boldsymbol{x}_0
+
\boldsymbol{\sigma}(t)\,\boldsymbol{z},
\label{eq:forward_unified_background}
\end{equation}
with different choices of $\boldsymbol{c}(t)$ and $\boldsymbol{\sigma}(t)$.
In the augmented trajectory view, time can be treated as an additional coordinate orthogonal to the sample dynamics~\cite{esser2024scaling}.

\subsection{Proofs for Section~\ref{sec:time_disentangled_diffusion}}

\begin{proof}[Proof of Proposition~\ref{prop:noisy_data_manifold}]
We consider the DDPM noise schedule as the representative case, since the analysis for flow matching proceeds in an analogous manner.
By ~\eqref{eq:forward_process}, every clean sample $\boldsymbol{x}_0\in\mathcal{M}_0$ can be written as
\begin{equation*}
    \boldsymbol{x}_0 = (x^0_1,\cdots,x^0_{d'},0,\cdots,0).
\end{equation*}
Therefore, for $\boldsymbol{x}_t = (x^t_1,\cdots,x^t_D)$,
\begin{equation*}
    \left\{
        \begin{aligned}
            x^t_i &= \sqrt{\bar{\alpha}_t}x^0_i + \sqrt{1-\bar{\alpha}_t}\, \boldsymbol{z}_i,\quad i=1,\cdots,d',\\
            x^t_i &= \sqrt{1-\bar{\alpha}_t}\, \boldsymbol{z}_i,\quad i=d'+1,\cdots,D.
        \end{aligned}
    \right.
\end{equation*}
It follows that
\begin{equation*}
    d(\boldsymbol{x}_t, \mathcal{M}_0)^2
    =
    \sum_{i=d'+1}^{D}\left(x^t_i\right)^2
    =
    (1-\bar{\alpha}_t)\sum_{i=d'+1}^{D}\boldsymbol{z}_i^2.
\end{equation*}
Because $\boldsymbol{z} \sim \mathcal{N}(\boldsymbol{0},\boldsymbol{I})$, the random variable $\sum_{i=d'+1}^{D}\boldsymbol{z}_i^2$ is $\chi^2$ distributed with $D-d'$ degrees of freedom \cite{hogg2013introduction}.
Applying the Laurent--Massart bound \citep{laurent2000adaptive} gives, for any $\varepsilon > 0$,
\begin{equation*}
    \mathbb{P}\Bigl(r(t)\sqrt{1-2\sqrt{\varepsilon}} \leq d(\boldsymbol{x}_t, \mathcal{M}_0) \leq r(t)\sqrt{1+2\sqrt{\varepsilon}+2\varepsilon}\Bigr)
    \geq
    1 - 2e^{-2(D-d')\varepsilon},
\end{equation*}
where $r(t)=\sqrt{(1-\bar{\alpha}_t)(D-d')}$.
Since $d'\ll D$, the concentration becomes sharp, so noisy data concentrate around the cylinder-like manifold defined in Equation~\eqref{eq:noisy_manifold}.
\end{proof}

\begin{proof}[Proof of Proposition~\ref{prop:exp_schedule}]
For the exp-$\sigma_t$ schedule,
\begin{equation*}
    \Delta \sigma_t
    = \sigma_{t+\Delta t}-\sigma_t
    = \frac{e^{t+\Delta t}-e^t}{e-1}
    = \left(\sigma_t+\frac{1}{e-1}\right)(e^{\Delta t}-1).
\end{equation*}
Since $r(t)=\sigma_t\sqrt{D-d'}$, we obtain
\begin{equation*}
    \Delta r(t)
    = \Delta \sigma_t\sqrt{D-d'}
    = \left(\sigma_t+\frac{1}{e-1}\right)(e^{\Delta t}-1)\sqrt{D-d'}.
\end{equation*}
If $e^{\Delta t}-1 > \kappa_{\varepsilon}$, then
\begin{equation*}
    \Delta r(t)
    > \kappa_{\varepsilon}\sigma_t\sqrt{D-d'}
    = \kappa_{\varepsilon} r(t),
\end{equation*}
which is exactly the sufficient condition in \eqref{eq:avoid_mixing}.
Therefore the exp-$\sigma_t$ schedule prevents overlap between adjacent noisy manifolds for all timesteps.
\end{proof}

\subsection{Experimental settings}

Across all experiments, the backbone and optimization settings are held fixed within each experimental regime, and only the time-conditioning mechanism or the forward schedule is changed.
This design ensures that the comparisons isolate the effect of manifold mixing versus manifold separation.

Throughout the experiments, we use five labels in the main text: None, Uniform-radial, Late-expansion, $\boldsymbol{t}_{ortho}$, and $t_{emb}$.
These correspond respectively to a model without explicit time conditioning under the standard variance-preserving trajectory, a time-unconditional model under the Uniform-radial VP schedule, a time-unconditional model under the Late-expansion VP schedule, the orthogonal time-space disentanglement model, and the standard model with explicit time embeddings.

\paragraph{Toy datasets.}
The toy experiments use 2D Gaussian mixtures and Swiss-roll data as clean manifolds.
To study ambient-dimension effects, each 2D sample is optionally zero-padded to a larger dimension $D \in \{2, 8, 32, 256, 1024\}$, so the first two coordinates contain the visible data and the remaining coordinates are exactly zero at $t=0$.
The default toy forward process uses $100$ diffusion steps, MSE training, Adam with learning rate $10^{-3}$ and betas $(0.9, 0.999)$, batch size $50{,}000$, and $200$ epochs.
Sampling uses DDIM~\cite{song2020denoising} with $\eta = 0$.
We compare a standard timestep-conditioned multilayer perceptron baseline with a time-unconditional manifold-aware toy network.
The latter is designed so that denoising of the visible coordinates depends on the geometry of the padded coordinates, thereby testing whether shell structure alone can reveal timestep information without an explicit temporal input.

\paragraph{CIFAR10 and CelebA with U-Net backbones.}
The small-image experiments use U-Net denoisers~\cite{ronneberger2015u} throughout.
CIFAR10 uses $32 \times 32$ RGB images, while CelebA uses $64 \times 64$ RGB images.
Unless otherwise stated, all models use $1000$ diffusion steps, $\epsilon$-prediction, MSE loss, EMA during training, and DDIM for sampling.
The main U-Net architecture uses base width $128$, channel multipliers $[1, 2, 2, 2]$, and two residual blocks at each resolution stage.
For CIFAR10, attention is inserted at an intermediate resolution stage, while for CelebA it is shifted to a slightly lower-resolution stage; this architecture is kept fixed across all compared variants.
The five compared settings are None, Uniform-radial, Late-expansion, $\boldsymbol{t}_{ortho}$, and $t_{emb}$.
The first three differ only in whether timestep information is absent and whether the forward radius follows the standard VP, Uniform-radial VP, or Late-expansion VP trajectory.
The $\boldsymbol{t}_{ortho}$ setting introduces an orthogonal time direction and a step-length parameter that controls the spacing between adjacent manifolds.
For class-conditional CIFAR10, $\boldsymbol{t}_{ortho}$ is instantiated with class-specific orthogonal directions so that both timestep and class identity are encoded geometrically.

For CIFAR10, the main training configurations use Adam with learning rate $2 \times 10^{-4}$, batch size $128$, gradient clipping $1.0$, learning-rate warmup $5000$, and EMA decay $0.9999$.
CelebA uses the same diffusion length and EMA setting with a lower learning rate $2 \times 10^{-5}$.
The class-conditional CIFAR10 experiments use $10$ classes.
The orthogonal time directions are estimated once from dataset statistics; for conditional generation, class-specific directions are estimated separately for each class.
All reported visual results use $100$ DDIM steps with $\eta = 0$.

\paragraph{ImageNet with diffusion transformers.}
The ImageNet experiments use latent diffusion transformers built on DiT-B/2~\cite{peebles2023scalable, rombach2022high} at $256 \times 256$ resolution.
Images are first mapped into latent tensors of shape $4 \times 32 \times 32$, and diffusion is performed in this latent space.
Training uses AdamW with learning rate $10^{-4}$, betas $(0.9, 0.999)$, zero weight decay, gradient clipping $1.0$, mixed-precision FP16, EMA updates, and distributed training through \texttt{accelerate}.
The effective per-process batch size is $32$.
We compare the same five settings as in the small-image experiments: None, Uniform-radial, Late-expansion, $\boldsymbol{t}_{ortho}$, and $t_{emb}$.
During training, classifier-free guidance~\cite{ho2022classifier, dhariwal2021diffusion} is enabled through label dropout with probability $0.1$.

For the $\boldsymbol{t}_{ortho}$ setting, the orthogonal time direction is estimated offline from latent features as the direction associated with the smallest mode of variation of the latent data matrix; class-conditional variants use the corresponding classwise estimate.
At inference, latent samples are decoded back to image space with the pretrained autoencoder, and the ImageNet results are reported with DDIM sampling, $\eta = 0$, and classifier-free guidance scale $5.5$.

\paragraph{Evaluation metrics.}
For the small-image experiments we report FID, Precision, and Recall.
For ImageNet we report FID, sFID, Inception Score, Precision, and Recall.
Across all three groups of experiments, only the schedule or the time-conditioning mechanism is changed, so the comparisons directly test the geometric claims made in Section~\ref{sec:time_disentangled_diffusion}.



\bibliographystyle{unsrtnat}
\bibliography{reference}



\end{document}